\documentclass[10pt]{amsart}
\UseRawInputEncoding

\usepackage{graphicx}
\usepackage{booktabs}
\usepackage{amsmath}
\usepackage{subcaption}

\usepackage{amsfonts}
\usepackage{amssymb}
\usepackage{amsmath}
\usepackage{amscd}
\usepackage{graphicx}
\usepackage{booktabs}
\usepackage{xcolor}
\usepackage{fancyhdr}
\usepackage{algorithm}
\usepackage{algpseudocode}
\usepackage{float}
\usepackage{mathrsfs}
\usepackage{stmaryrd}

\usepackage[colorlinks,linkcolor=blue,citecolor=blue,urlcolor=blue]{hyperref}

\newcommand{\beq}{\begin{equation}}
\newcommand{\eeq}{\end{equation}}
\newcommand{\bea}{\begin{eqnarray}}
\newcommand{\eea}{\end{eqnarray}}
\newcommand{\nn}{\nonumber}
\newcommand{\noi}{\noindent}
\newcommand{\bos}{\boldsymbol}

\newtheorem{assumption}{Assumption}

\newcommand{\bi}[1]{\textbf{\textit{#1}}}

\newcommand{\bln}{\rm{\bf{ln}}}
\newcommand{\bexp}{\rm{\bf{exp}}}
\newcommand{\blog}{\rm{\bf{log}}}
\newcommand{\bof}{\boldsymbol{f}}
\newcommand{\R}{\mathbb{R}}
\newcommand{\Rp}{\mathbb{R}^{N}_{+}}

\newcommand{\bg}{\boldsymbol{g}}
\newcommand{\bQ}{\boldsymbol{Q}}
\newcommand{\bH}{\boldsymbol{H}}
\newcommand{\bU}{\boldsymbol{U}}

\newcommand{\bu}{\boldsymbol{u}}
\newcommand{\bx}{\boldsymbol{x}}
\newcommand{\by}{\boldsymbol{y}}
\newcommand{\bv}{\boldsymbol{v}}
\newcommand{\bw}{\boldsymbol{w}}

\newcommand{\RR}{\mathbb{R}}

\newtheorem{definition}{Definition}
\newtheorem{proposition}{Proposition}
\newtheorem{theorem}{Theorem}

\newtheorem{lemma}{Lemma}
\theoremstyle{definition}

\newtheorem{remark}{\textbf{Remark}}

\theoremstyle{definition}	 %AÃ±adimos esto para que tenga el estilo de una definiciÃ³n y no la de teorema o proposiciÃ³n. Por defecto el estilo es "plain", que es el de teoremas.

\usepackage{lastpage}
\usepackage{mathtools}
\usepackage{mathrsfs}
\usepackage{enumitem}
\usepackage{stmaryrd}
\usepackage{empheq,etoolbox}
\usepackage{xcolor}

\def\BibTeX{{\rm B\kern-.05em{\sc i\kern-.025em b}\kern-.08em
    T\kern-.1667em\lower.7ex\hbox{E}\kern-.125emX}}
\usepackage[active]{srcltx} %SRC Specials for DVI Searching

\usepackage{bm}
\usepackage{textcomp}
\usepackage{empheq}
\usepackage{booktabs}

\begin{document}

\title[Group Entropies  and Mirror Duality]{Group Entropies  and Mirror Duality:\\
A  Class of Flexible  Mirror Descent Updates \\ for Machine Learning}

\author{Andrzej Cichocki} 
\address{Systems Research Institute of Polish Academy of Science\\
      Newelska 8, 01-447 Warszawa, Poland, \\ Warsaw University of Technology \\ Koszykowa 75, 00-662 Warszawa, Poland}
\email{cichockiand@gmail.com}
   
\author{Piergiulio Tempesta}
\address{Departamento de F\'{\i}sica Te\'{o}rica, Facultad de Ciencias F\'{\i}sicas \\
  Universidad Complutense de Madrid, 28040 -- Madrid, Spain \\ and Instituto de Ciencias
  Matem\'aticas \\ C/ Nicol\'as Cabrera, No 13--15, 28049 Madrid, Spain}
\email{p.tempesta@fis.ucm.es}


%\editor{}

\date{March 8, 2026}

\maketitle

\begin{abstract}
We introduce a comprehensive theoretical and algorithmic framework that bridges formal group theory and group entropies with modern machine learning, paving the way for an infinite, flexible family of Mirror Descent (MD) optimization algorithms. Our approach exploits the rich structure of group entropies, which are generalized entropic functionals governed by group composition laws, encompassing and significantly extending all trace-form entropies such as the Shannon, Tsallis, and Kaniadakis families. By leveraging group-theoretical mirror maps (or link functions) in MD, expressed via multi-parametric generalized logarithms and their inverses (group exponentials), we achieve highly flexible and adaptable MD updates that can be tailored to diverse data geometries and statistical distributions. To this end, we introduce the notion of \textit{mirror duality}, which allows us to seamlessly switch or interchange group-theoretical link functions with their inverses, subject to specific learning rate constraints. By tuning or learning the hyperparameters of the group logarithms enables us to adapt the model to the statistical properties of the training distribution, while simultaneously ensuring desirable convergence characteristics via fine-tuning. This generality not only provides greater flexibility and improved convergence properties, but also opens new perspectives for applications in machine learning and deep learning by expanding the design of regularizers and natural gradient algorithms. We extensively evaluate the validity, robustness, and performance of the proposed updates on large-scale, simplex-constrained quadratic programming problems.

\end{abstract}

\tableofcontents

\section{Introduction}
The purpose of this article is to merge formal group theory with
machine learning techniques with the aim of constructing novel mirror descent algorithms that possess suitable group-theoretical properties, making them flexible tools for optimization problems.

A variety of fundamental gradient-based update schemes have been thoroughly investigated within the domains of optimization, machine learning, and artificial intelligence. These include additive gradient descent (GD) \cite{Nemirowsky,EGSD}, its stochastic variant (SGD), multiplicative updates (MU) \cite{Cia3,Cichocki_Cruces_Amari,CichZd_ICA06}, exponentiated gradient (EG) descent \cite{EG,MD2}, and mirror descent (MD) algorithms \cite{Nemirowsky,MD1,Beck2003,EGSD,shalev2011}.

The additive gradient descent (GD) method and its multiplicative counterpart, the exponentiated gradient (EG) update, are among the most commonly used algorithms in machine learning. We recall that EG descent updates, which are a subclass of both multiplicative gradient schemes \cite{Cia3} and mirror descent (MD) schemes (e.g., \cite{MD1,EGSD}), were introduced by  Kivinen and Warmuth \cite{EG,KW1995}. EG algorithms have since been extended and applied across diverse domains by numerous researchers \cite{Helmbold98,herbsterwarmuth,EGnoise,Nock2023,Yang2022,Momentum,EGSD}.

Mirror descent is a versatile optimization framework, prevalent in optimization, AI, machine learning, and online learning, that generalizes classical gradient descent to non-Euclidean settings by means of a mirror map (or potential function). Rather than using the standard Euclidean norm, mirror descent employs a Bregman divergence, that serves as a proximal regularizer for the optimized loss function.

A crucial aspect is that the geometry of the problem, in the mirror gradient descent approach, is incorporated by the strictly convex potential (or generating) function $F(\bw)$. By choosing flexible potential functions, mirror descent can be adapted to
the specific optimization problem,  in order to ensure faster convergence and/or better robustness. Generally speaking, this makes mirror descent superior to standard gradient descent optimization methods.
Observe that by adopting $F(\bw) =(1/2)||\bw||_2^2$ as a distance-generating function, then the Bregman distance reduces to the Euclidean distance, and mirror descent to standard gradient descent. By using parameterized potential functions,  we can adapt to  different geometries,  such as the simplex or more general convex sets.

Exponentiated Gradient  descent and Mirror Descent   have seen widespread adoption across numerous optimization and machine learning applications, and have also been successfully applied to non-convex problems \cite{Nemirowsky,Beck2003,MD1,MD2,EGSD,shalev2011}.

Although additive gradient descent remains the predominant optimization method, it is often ill-suited for problems requiring all components of the weight vector to remain nonnegative. Furthermore, when employing additive gradient updates, practitioners frequently encounter vanishing or exploding gradients, which require meticulous learning-rate selection and tuning to ensure stable and effective training in real-world applications \cite{Nemirowsky}.

Interestingly enough, EG weight updates can mitigate many of the issues inherent to additive GD, without the need for finely tuned learning rates. In particular, when the target weight vectors are sparse, the un-normalized exponentiated gradient (EGU) and normalized exponentiated gradient variants have been shown to converge more rapidly than standard GD \cite{EG,herbsterwarmuth,MD1}.

%Furthermore, recent neurophysiological studies of synaptic plasticity suggest that multiplicative EG-style update mechanisms more closely mirror biologically observed learning processes than do conventional additive GD updates \cite{Cornford24}.

Conventional exponentiated gradient (EG) updates are most naturally suited to domains whose geometry is governed by the Kullback-Leibler divergence, as is common when optimizing over probability distributions. However, a significant drawback of these classic EG schemes is their rigidity: the absence of adjustable hyperparameters prevents them from adapting to datasets with diverse distributional properties, thereby constraining both their convergence dynamics and practical performance.

The central contribution of this work is a general theoretical framework which establishes a rigorous connection between machine learning and group theory. In fact, in order to take into account the geometry of the problem under consideration, we design exponentiated gradient updates based on multi-parametric link functions grounded in group-theoretic principles.

More precisely, the theory of formal groups \cite{Haze,Serre}, a fundamental branch of algebraic topology, offers a rich source of link functions expressed in terms of generalized logarithms and exponentials. These generalized functions, in turn,  find their primary physical motivation in the theory of \textit{group entropies}.

This approach was first introduced by one of us in \cite{PT2011PRE}, further developed in \cite{tempesta2015,PT2016AOP,tempesta2016,PT2020CHA} and subsequently explored by various authors (see \cite{RRT2019PRSA,TJ2020SR,ADT2021CHA,ADTCHA2022} and \cite{JT2018ENT,JT2024ENT} for reviews). The notion of group entropy represents an axiomatic formulation of the theory of generalized entropies, based on the first three Shannon-Khinchin axioms (continuity of entropy, maximum entropy principle, expansibility) and the new composability axiom formulated in \cite{PT2016AOP}.  These requirements are necessary to ensure that a given entropy can be suitable for information-theoretical and thermodynamic purposes.

From a mathematical point of view, the composability axiom amounts to saying that the entropy $S(A\times B)$ of a system composed of two systems $A$ and $B$ has to satisfy a group law. For instance, the standard Boltzmann-Gibbs entropy satisfies the additivity law; the Tsallis entropy fulfils the so-called multiplicative law $S(A\times B)= S(A) + S(B)+ (1-q)S(A)S(B)$, and so on.

In this work, we present two main results.

First, we shall prove that the class of known mirror descent algorithms can be broadly generalized through the theory of group entropies. In fact, this theory naturally yields \textit{infinitely many multiparametric
generalized logarithms}, that we will interpret as \textit{group-theoretical link functions}. We also propose a novel construction of group logarithms from known ones, and introduce the notion of \textit{chain link functions}. They represent a novel and versatile tool which adapts to the features of the optimization problem under consideration. A key property is that, in an appropriate limit, they reduce to the standard natural logarithm.
This enhancement guarantees more flexibility to adapt to data within different geometries and probability distributions.

\vspace{2mm}

The second main result is a contribution to the theory of MD algorithms for ML, based on a crucial duality of our group-theoretical approach.

In order to clarify this concept, we recall that the fundamental requirement for the MD updates is that the generating (potential) function $F(\bw)$ be a twice-differentiable, strictly convex function on a convex domain $\mathcal{C}\in \R^{N}$. This in turn requires that the link function $\bof(\bw):=\nabla F(\bw)$ should be a strictly increasing and invertible (componentwise) function.
Moreover, when the link function is concave (as in the case of group logarithms) the growth rate of the associated Bregman divergence is reduced, leading to a decrease in geometric curvature. This property improves the stability of the resulting algorithms, albeit potentially at the expense of a slower convergence rate. Conversely, selecting a convex link function (such as group exponential functions) induces a geometric increase in the growth rate of the Bregman divergence, corresponding to higher geometric curvature, and may therefore enable faster convergence.

Thus, assume that the link function is expressed by a concave group logarithm. Then, we have the following invariance property.

\vspace{2mm}

\textbf{Mirror Duality}: The MD updates can be formulated by selecting as a link function either a group logarithm or the associated group exponential. All formulas stay valid under this interchange under imposing some constraints on learning rates. The resulting theories will be said to be mirror-dual.

\vspace{2mm}

Needless to say, this symmetry is intrinsic to the theory and amounts to the interchange of the link function with its inverse. However, it acquires a novel geometrical meaning in the context of group-theoretical logarithms and exponentials, since it also implies a ``geometry interchange '' among dual theories. We believe that this duality could be crucial in many applicative contexts.

Using the mirror duality, we propose a novel class of algorithms, which we call the Dual Mirror Descent (DMD). We study numerically the behaviour of these algorithms in a relevant and concrete  Simplex constrained Optimization problems. We will show that DMD algorithms outperforms the standard EG and the generalized EG algorithms in several of the scenarios we tested.

In essence, we propose a general group-theoretical foundation for mirror descent analysis, which admits potential applications across different research domains, including statistical mechanics and portfolio optimization.

This article is organized as follows. In Section \ref{Sec2}, we review basic notions for the subsequent discussion, as formal groups, the Shannon-Khinchin axiomatic formulation and the theory of group entropies. In Section \ref{Sec3}, we introduce the machine learning algorithms object of our analysis, and specify our notation. In Section \ref{Sec4}, we propose a  family of mirror descent algorithms based on group theory.   Section \ref{Sec5} is devoted to computer simulation experiments of our novel algorithms, whereas Section \ref{Sec6} provides  a theoretical analysis of the main stability,  convergence and robustness properties of Dual Mirror Descent algorithms. In the final Section \ref{Sec7}, some conclusions are drawn and future research perspectives are outlined.

\section{Preliminaries: Groups and Entropies.} \label{Sec2}

\subsection{Group Entropies: A Brief Review}

In order to make the subsequent discussion self-contained, first we shall briefly summarize the main results of the theory of group entropies. The motivation is to provide an axiomatic formulation
of the theory of generalized entropies, that would allow us to connect them to information theory and statistical mechanics.

We recall that the Shannon-Khinchin axioms were proposed independently by Shannon and Khinchin as properties characterizing uniquely the mathematical form of Boltzmann's entropy as a function
$S(p_1,\ldots,p_W)$ in the space of probability distributions. They can be stated as follows.

(SK1) (Continuity). The function $S(p_1,\ldots,p_W)$ is continuous with respect to all its arguments

(SK2) (Maximum principle).  The function $S(p_1,\ldots,p_W)$ takes its maximum value for the uniform distribution $p_i=1/W$, $i=1,\ldots,W$.

(SK3) (Expansibility). Adding an impossible event to a probability distribution does not change its entropy: $S(p_1,\ldots,p_W,0)=S(p_1,\ldots,p_W)$.

(SK4) (Additivity). Given two subsystems $A$, $B$ of a statistical system,
\[
S(A \times B)=S(A)+S(B\mid A).
\]
It is clear that, by relaxing one of these axioms, new possibilities arise. The first three axioms are very natural properties, which are crucial both in information theory and statistical mechanics.
 Let $\mathcal{P}_W$ denote the set of all probability distributions of the form $p=(p_1,\ldots,p_W)$, $p_i\geq 0$, $\sum_i p_i =1$.
\begin{definition}
A function $S: \mathcal{P}_W \to \mathbb{R}_{\geq0}$ that satisfies the axioms (SK1)--(SK3) will be said to be an entropic function (or a generalized entropy).
\end{definition}

Instead, the additivity axiom (SK4) has been replaced in \cite{PT2016AOP}, \cite{tempesta2016} by a more general requirement, called the \textit{composability axiom}.
Precisely, given two \textit{statistically independent} systems $A$ and $B$, we require that the entropy satisfies the composition law
\begin{equation}
S(A\times B)= \Phi(S(A), S(B)),
\end{equation}
where $\Phi$ should fulfill some properties in order to have a physical meaning:
\begin{enumerate}
\item \label{C} $\Phi(x,y)=\Phi(y,x)$ ({\it commutativity}),
\item  $\Phi(x,0)=x$ ({\it identity)},
\item  \label{A} $\Phi(x,\Phi(y,z))=\Phi(\Phi(x,y),z)$ ({\it associativity)}.
\end{enumerate}
 These properties are necessary to ensure that a given entropy may be suitable for information-theoretical and
thermodynamic purposes. Indeed, it should be symmetric with respect
to the exchange of the labels $A$ and $B$ (i.e., commutativity). Furthermore, if we compose a given
system with another system in a state with zero entropy, the total entropy should
coincide with that of the given system (i.e., identity). Finally, the composability of more than two
independent systems in an associative way is crucial to ensure path independence in the composition.

If the composability axiom is satisfied in full generality, namely for all possible distributions defined in the space $\mathcal{P}_W$, the corresponding entropy will be said to be
\emph{strictly composable}.
If it is satisfied on the uniform distribution only, the entropy is \textit{weakly composable}.
\begin{definition}
A group entropy is an entropy that satisfies the first three Shannon-Khinchin axioms and is strictly composable.
\end{definition}
In \cite{ET2017JSTAT} it has been proved that the trace-form class of entropies \textit{cannot} be strictly composable, with the notable exception of Tsallis entropy. Therefore, non-trace-form entropies are needed to satisfy the composability axiom. For this reason, in \cite{tempesta2016} the class of Z-entropies has been proposed:
\begin{equation}  \label{eq:Z}
Z_{G,\alpha }(p):=\frac{1}{1-\alpha }\log _{G}\Big(\sum_{i=1}^{W}p_{i}^{%
\alpha }\Big).
\end{equation}
Here $\log_{G}(x)$ denotes a deformation of the Neperian logarithm, known as the \textit{group logarithm} \cite{PT2011PRE}, which will be defined below.
The entropies of the class \eqref{eq:Z} are group entropies which, in addition, satisfy the fundamental set of Shore-Johnson axioms \cite{SJ1,SJ2}, recently studied and characterized in \cite{JK2019}.
The notion of group logarithm, and its inverse, the \textit{group exponential} $\exp_{G}(x)$ will be central objects of our investigation.

\subsection{Formal Groups: Main Definitions}
First, we shall review some results concerning formal group theory and its connection to the theory of generalized entropies, closely following  refs. \cite{PT2016AOP} and \cite{tempesta2016}.

%\subsection*{}
Since Bochner's seminal work  \cite{Bochner1946}, the theory of formal groups has found numerous applications, ranging from the original algebraic-topological context to the theory of elliptic curves, arithmetic number theory, and combinatorics. We begin by recalling some basic facts and definitions of the theory of formal groups (see \cite{Haze}, \cite{Serre}  for a thorough exposition).

Let $R$ be a commutative associative ring  with identity, and $R\left\{ x_{1},\text{ }%
x_{2},..\right\} $ be the ring of formal power series in the variables $x_{1}$, $x_{2}$,
... with coefficients in $R$.

\begin{definition} \label{formalgrouplaw}
A commutative one--dimensional formal group law
over $R$ is a formal power series $\Phi \left( x,y\right) \in R\left\{
x,y\right\} $ such that \cite{Bochner1946}
\begin{equation*} \label{compositionlaw}
1)\qquad \Phi \left( x,0\right) =\Phi \left( 0,x\right) =x,
\end{equation*}%
\begin{equation*}
2)\qquad \Phi \left( \Phi \left( x,y\right) ,z\right) =\Phi \left( x,\Phi
\left( y,z\right) \right) \text{.}
\end{equation*}
When $\Phi \left( x,y\right) =\Phi \left( y,x\right) $, the formal group law is
said to be commutative.
\end{definition}
The existence of an inverse formal series $\tilde{\varphi}
\left( x\right) $ $\in R\left\{ x\right\} $ such that $\Phi \left( x,\tilde{\varphi}
\left( x\right) \right) =0$ is a consequence of Definition \ref{formalgrouplaw}. Let  $B=\mathbb{Z}[b_{1},b_{2},...]$ be the  ring of integral polynomials in infinitely many variables. We shall consider the series $
H\left( s\right) = \sum_{i=0}^{\infty} b_i \frac{s^{i+1}}{i+1}$,
with $b_0=1$. Let $G\left( t\right)$ be its compositional inverse:
\beq \label{I.2}
G\left( t\right) =\sum_{k=0}^{\infty} a_k \frac{t^{k+1}}{k+1},
\eeq
 i.e. $H\left( G\left( t\right) \right) =t$. Thus, we shall use the notation $H=G^{-1}$. From this property, we deduce $a_{0}=1, a_{1}=-b_1, a_2= \frac{3}{2} b_1^2 -b_2,\ldots$.
%\begin{definition}
The Lazard formal group law \cite{Haze} is defined by the formal power series
\[
\Phi_{\mathcal{L}} \left(x, y\right) =G\left( G^{-1}\left(
x\right) +G^{-1}\left(y\right) \right).
\]
%\end{definition}
The coefficients of the power series $G\left( G^{-1}\left( x\right) +G^{-1}\left(y\right) \right)$ lie in the ring $B \otimes \mathbb{Q}$ and generate over $\mathbb{Z}$ a subring $A \subset B \otimes \mathbb{Q}$, called the Lazard ring $L$.

For any commutative one-dimensional formal group law over any ring $R$, there exists a unique homomorphism $L\to R$ under which the Lazard group law is mapped into the given group law (the \textit{universal property} of the Lazard group).

Let $R$ be a ring with no torsion. Then, for any commutative one-dimensional formal group law $\Phi(x,y)$ over $R$, there exists a series $\psi(x)\in R[[x]] \otimes \mathbb{Q}$ such that
\[
\phi(x)= x+ O(x^2), \quad \text{and} \quad \Phi(x,y)= \phi^{-1}\left(\phi(x)+\phi(y)\right)\in R[[x,y]]\otimes \mathbb{Q}.
\]
The universal formal group plays the role of the general composition law admissible for the construction of the family of group entropies.

We also recall that a second composition law, compatible with the law \eqref{compositionlaw} was proposed in \cite{CT2019prep}, allowing the definition of \textit{formal rings}.

\subsection{Generalized Logarithms and Exponentials from Group Laws}
There is a simple construction allowing to define a generalized logarithm from a given group law.
\begin{definition} \label{defglog}
Let $G$ be a series of the form \eqref{I.2}. A group logarithm is a continuous, strictly concave, monotonically increasing function $\log_{G}: (0,\infty)\to \mathbb{R}$, possibly depending on a set of real parameters, such that $\log_{G}(\cdot)$ solves the functional equation for the group law corresponding to $G$, i.e
\beq
\log_{G}(xy)= \Phi(\log_{G}(x),\log_{G}(y)), \label{glog}
\eeq
where
\beq
\Phi(x,y)=G(G^{-1}(x)+G^{-1}(y)). \label{GPsi}
\eeq
\end{definition}

%\begin{theorem}
\begin{remark}
A simple way to realize a group logarithm is the following: one can introduce
\begin{equation} \label{Glog}
\log_{G}(x):=G(\ln (x)) \ ,
\end{equation}
where $G\in C^{1}(\mathbb{R}_{\geq 0})$ is a strictly increasing function, with  $G(t)=t+O(t^2)$,   taking positive values over $\mathbb{R}_{+}$. Then $G(\ln (x))$  is at least a Schur-concave function and it satisfies  Eq. \eqref{glog}, where $\Phi(x,y)$ is the group law \eqref{GPsi}.
\end{remark}

The theory, in an abstract sense, can also be formulated in a field of characteristic zero in the setting of formal power series. As is well known \cite{Haze}, for a one-dimensional formal group law $\Phi(x,y)$ over a  torsion-free ring, there exists a formal series $G(t)$ of the form \eqref{I.2} that satisfies Eq.. \eqref{GPsi}. %Then the same relations \eqref{proofTh1} still hold.

\noi As an example, when $\Phi(x,y)=x+y$, we immediately obtain that $G(t)=t$ and $\log_{G}(x)=\ln (x)$. If instead
$\Phi(x,y)=x+y+(1-q)xy$, the  associated function $G(t)$ is given by $G(t)=\frac{e^{(1-q)t}-1}{1-q}$ and the group logarithm reduced to the  Tsallis $q$-logarithm
\beq
\log_q(x):=\frac{x^{1-q}-1}{1-q},  \;\; x>0, \;\; q>0, \;\; q \neq 1. \label{Tslog}
\eeq
with the inverse function $\exp_q(x)= [1+(1-q)x]^{1/(1-q)}_+$.
\begin{remark}
If $G(t)$ is given by the general series \eqref{I.2}, the concavity of $\log_{G}(x)$ is guaranteed by the condition
\beq
 a_{k} >    (k+1)  a_{k+1} \qquad \forall k\in\mathbb{N} \qquad \text{with } \{a_{k}\}_{k\in\mathbb{N}}\geq 0 \label{conc},
\eeq
which is also sufficient to ensure that the series $G(t)$ is absolutely and uniformly convergent with radius of convergence $r=\infty$.
\end{remark}

Thus, under mild assumptions, given a group law,  a group logarithm can be obtained from relation \eqref{Glog} together with the condition \eqref{conc}.  A construction of group logarithms from difference operators has been proposed in \cite{PT2011PRE}.

Note that the theory outlined above allows to consider group logarithms with an arbitrary number of hyper-parameters (as we shall show in Section 4).

\begin{definition} \label{defgexp}
The inverse of a group logarithm is said to be the group exponential; it is defined by
\beq
\exp_{G}(x)= e^{G^{-1}(x)}. \label{Gexp}
\eeq
\end{definition}
When $G(t)=t$, we have back the standard exponential; when $G(t)=\frac{e^{(1-q)t}-1}{1-q}$, we recover the $q$-exponential $e_{q}(x)=\left[1+(1-q)t\right]^{\frac{1}{1-q}}$, and so on.
\begin{remark}
The Lagrange inversion principle allows to compute the formal compositional inverse as the series %$G^{-1}(s)$, such that $G(G^{-1}(s))=s$ and G^{-1}(G(t))=t$ can be obtained by means of %the Lagrange inversion theorem. We get the formal power series
\beq \nn
G^{-1}(x)=x-\frac{a_1}{2}x^2+ \frac{1}{6}(3 a_1^2-2 a_2)x^3+\frac{1}{24}(-15 a_1^2+20 a_1 a_2-6 a_3)x^4 \ldots
\eeq
A priori, when a closed-form expression is not available, we can compute the group exponentials at least as the formal power series inverse of $\log_{G}(x)$.
%Then, from the relation $\Phi(x,y)= G\left(G^{-1}(x)+G^{-1}(y)\right)$, with $\Phi(x,y)=x+y+ \text{higher order terms}$, we get a system of equations that allows to reconstruct the sequence $\{a_{k}\}_{k\in\mathbb{N}}$.

%Each element $a_{k}$  a priori depends on the set of parameters appearing in $\Psi(x,y)$.
\end{remark}

%\subsection{Fundamental Properties of Group Logarithms and Exponentials}

The group logarithms have following fundamental properties (which follows from SK axioms)

  \begin{itemize}

 \item

Strictly monotonically increasing function: $\displaystyle \frac{d \log_G(x)}{dx} >0$ \;\; (for a wide range of parameters, this property is essential and necessary for MD) \\

 \item

 Concavity:  $\displaystyle \frac{d^2 \log_{G}(x)}{dx^2} < 0$ \;\; (for a limited range of parameters).

 \end{itemize}

 Analogously, group-exponential functions have the following basic properties

 \begin{itemize}

% \item

% Domain $\exp{G}(x) $:  $\Real \rightarrow \Real^+$\\

% \item

% Scaling and Normalization: $\exp{G}(0)=1$, \;\;  $\displaystyle \frac{d \exp_{G}(x)}{dx}\Bigg|_{x=0} =1$\\

 \item

Strictly monotonically increasing function: $\displaystyle \frac{d \exp_G(x)}{dx} >0$ \;\; (for a wide range of parameters) \\

 \item

Convexity:  $\displaystyle \frac{d^2 \exp_{G}(x)}{dx^2} > 0$ \;\; (for a limited range of parameters)\\

% \item

%  $\exp_{G}(x) \exp_G(-x)=1$.

 \end{itemize}

\subsection{Group Logarithms: Examples}
To compare the present general theoretical approach with the existing literature, we observe that in fact, the known cases can be easily accommodated within the group entropy theory. From the general definition $Log_{G}(x):=G(\ln (x))$, we obtain the following classification.

\vspace{3mm}

a) The standard Neperian logarithm is recovered when $G(t)=t$.

b) The Tsallis logarithm \cite{tsallis1988} is obtained for $G(t)= \frac{e^{(1-q)t}-1}{(1-q)}$.

c) The Kaniadakis logarithm \cite{kaniadakis2002,kaniadakiseditorial2004,kaniadakis2005,
kaniadakis2009,kaniadakis2017} corresponds to the choice
$G(t)= \frac{\text{exp}(\kappa t)-\text{exp}(-\kappa t)}{2 \kappa}$

d) The Euler generalized logarithm
\beq
\log_{Euler}(x):=\frac{x^{a}-x^{b}}{a-b}.
\eeq
recently studied in \cite{Cichocki2025a,Cichocki2025c} is related with the Abel formal group. It can be interpreted as a group logarithm obtained via the function
\beq
\nn
\text{G}_{Abel}(t)=\frac{e^{at}-e^{bt}}{a-b}=\frac{e^{\gamma t}}{\sqrt{\delta}}\sinh \left(\sqrt{\delta} t\right), \label{Abelexp}
\eeq
where
$
\gamma= (a+b)/2$, $\delta=(a-b)^2/4.$
This function satisfies the celebrated Abel functional equation, as proved by Abel itself.
%The $Z_{a,b}$-entropy is related to the Borges-Roditi logarithm:

e) The family of generalized logarithms introduced in \cite{tempesta2015} and recently studied in \cite{Cichocki2025b} is generated by the expansion
\beq
\nn G(t)=t-\sigma\frac{a^2 \varphi''(a)+a \varphi'(a)-1}{2(-1+a\varphi'(a))}t^2 +
\sigma^2\frac{a^3 \varphi'''(a)+3 a^2 \varphi''(a)+a \varphi'(a)-1}{6(-1+a\varphi'(a))} t^3+\ldots
\eeq
Notice that $G(t)$ is invertible, and its inverse $G^{-1}(t)$ can be computed term by term with the desired precision by means of the Lagrange inversion formula. In some specific cases, it converts into a closed-form function.

The construction in \cite{tempesta2015} allows to introduce an infinite family of generalized logarithms. We first introduce the notion of admissible function.
\begin{definition}\label{admis}
Let $\varphi:[0,1)\to\mathbb{R}$ be a $\mathcal{C}^{\infty}(0,1)$ function, continuous in $[0,1)$, such that
\beq
\frac{1-a \varphi'(a x^{\sigma})-a^2  x^{\sigma} \frac{\sigma}{1+\sigma}  \varphi''(a x^{\sigma}) }{1-a \varphi'(a)}> 0 \label{ineq}
\eeq
for $x\in[0,1]$, $a\in (0,l_a)$, $\sigma\in (0,l_{\sigma})$, with $l_a,l_{\sigma}\in(0,1)$. Then $\varphi$ will be said to be an admissible function.
\end{definition}
We shall denote by $\mathcal{A}$ the set of admissible functions.
Then we can define an infinite family of group logarithms of the form
\beq \label{eq:templog}
\log_{\varphi}(x):=\frac{1}{1-a \varphi'(a)}\left(\frac{\varphi(ax^{-\sigma})-x^{-\sigma}}{\sigma}+\frac{1-\varphi(a)}{\sigma}\right)
\eeq
where $\varphi\in \mathcal{A}$.
We can illustrate this formulation of group logarithms by a simple example: Assuming that
\beq \nn
\varphi(x) = x^r \left[ (\lambda x)^{\kappa} - (\lambda x)^{-\kappa} \right] +x
= \lambda^{\kappa} x^{r+\kappa} -\lambda^{-\kappa} x^{r-\kappa} +x
\eeq
and taking $\alpha=1$ and $\sigma=-1$, we obtain the three parameters  generalized logarithm investigated by Kaniadakis  \cite{kaniadakis2009}
\beq
\displaystyle \log_{\kappa,r,\lambda} (x) = \frac{\lambda^{\kappa} x^{r+\kappa} - \lambda^{-\kappa} x^{r-\kappa} - \lambda^{\kappa} + \lambda^{-\kappa} }{(r+\kappa) \lambda^{\kappa} - (r-\kappa) \lambda^{-\kappa}}, \;  x>0, \;\; \lambda >0, \kappa \in [-1,1],  \; -|\kappa| < r < |\kappa|.
\eeq
Many other  examples of group logarithms and related entropies of this class are provided in \cite{tempesta2015} and \cite{Cichocki2025b}. It is interesting to mention the simple group logarithms obtained by replacing into Eq. \eqref{eq:templog} the admissible functions $\varphi(x)=\exp (x)$ or $\varphi(x)= \sin (x)$, and the notable logarithms obtained from Dirichlet series of number theory.

\vspace{3mm}

\subsection{The Super--Exponential Case} As alternative and nontrivial examples, we propose the family of super-exponential entropies, which are designed to be extensive when $W(N)$ grows faster than exponentially. The first known example is defined by
	\begin{equation}
	S[p]=\lambda\Bigg\{\exp\left[\mathcal{L}\Big(\frac{\ln \sum_{i=1}^{W(N)}p_i^\alpha}{\gamma(1-\alpha)}\Big)\Bigg]-1\right\}, \label{superexpEnt}
	\end{equation}
where $p=(p_1,\ldots,p_W)\in \mathcal{P}_W$ and $\mathcal{L}$ denotes the principal branch of the Lambert function \footnote{The principal branch is defined as the unique solution of the functional relation $\mathcal{L}(x)e^{\mathcal{L}(x)}=x$ for $x\geq 0$.}. This super-exponential entropy  was recently introduced in relation to a simple model whose components can form emergent paired states in addition to the combination of single-particle states \cite{JPPT2018JPA}. The group-theoretical structure is as follows. From the function
\beq
G(t)=(1-\alpha)\Bigg\{\exp\Big[\frac{\mathcal{L}(t)}{\gamma(1-\alpha)}\Big]-1\Bigg\},
\eeq
the associated group logarithm reads:
\beq
\log_G(x):=(1-\alpha)\Bigg\{\exp\Big[\frac{\mathcal{L}(\ln (x))}{\gamma(1-\alpha)}\Big]-1\Bigg\},
\eeq
with $\alpha>0$, $\gamma\geq 1$.

\newpage

%Another interesting case is provided by the \textit{stretched-exponential entropy}
%\begin{equation}
%S_{\alpha, \gamma}[p[ = \ {k} \left[ \frac{\ln \big( \sum_{i=1}^{W} p^{\alpha}_i\big)}{1-\alpha}\right]^{\!1/\gamma} \, .
%\end{equation}
%This entropy has been introduced by one of us in \cite{PT2020CHA} and recently applied, in \cite{TJPLB2024}, to the analysis of cosmological data.
%We get
%\beq
%\log_{\alpha, \gamma}(x)= (1-\alpha)\Big(\frac{\ln x}{1-\alpha}\Big)^{\frac{1}{\gamma}},
%\qquad \exp_{\alpha, \gamma}(x)=\exp\Big((1-\alpha)\Big(\frac{x}{1-\alpha}\Big)^{\gamma}\Big),
%\eeq
%with  $\alpha<1$. Furthermore, for $0<x<1$ we assume $\gamma<1$; for $x>1$, we take $\gamma>1$ in order to ensure concavity of $\log_{G}(x)$ (convexity of $\exp_{G}(x)$).

\subsection {On the Relevance of Group Entropies in Machine Learning and Information Theory}
% \tcb{(Revised subsection)}
Group entropies are a broader class of entropies constructed from formal group laws, satisfying the first three Shannon-Khinchin axioms and a composability axiom. Group entropies provide a systematic, axiomatic, and highly flexible framework for generating new entropies, with significant advantages for potential applications in  machine learning and deep learning applications, especially in robust optimization and adaptive algorithms.
Since they also satisfy the Shore-Johnson axioms (unlike trace-form entropies, except for Tsallis entropy), they represent at the same time proper information measures and are potential candidates for thermodynamic applications.

Group entropies provide us with a large class of group logarithms and group exponentials, which are the building blocks for our construction of mirror descent algorithms.
%  All trace-form entropies can be considered  as special cases, but also allows for more general, non-trace-form functionals.
In summary, the main features of our construction are the following.

\begin{itemize}

%\item Trace-form entropies which have been mostly applied till now in machine learning are only a subset of group entropies.

   \item  {\it Systematic Generation}: The group-theoretical approach allows for the systematic and iterative construction of new entropies by combining or composing existing group laws.

\item {\it Multi-Parametric Flexibility}: Group entropies can depend on several hyperparameters, which can be tuned or learned to adapt to the statistical properties of data or the geometry of the optimization problem.

\item {\it Composability}: The composability axiom ensures that the entropy is well-behaved under the combination of independent systems, which is crucial for information theory and thermodynamics.

%\item {\it Unification}: Trace-form entropies (Shannon, Tsallis, Kaniadakis, Sharma-Mittal, etc.) are all special cases of group entropies, but group entropies can go beyond trace-form (e.g., super-exponential, Z-entropies).

\item {\it Adaptability in Machine Learning}: The flexibility in choosing the group law and parameters allows for the design of loss functions and regularizers that are robust to outliers, adapt to data geometry, and potentially improve convergence in optimization algorithms (e.g., mirror descent, exponentiated gradient updates).

\item {\it Interpretability}: Some group entropies may lack a clear physical or information-theoretic interpretation, especially when generated by highly nontrivial group laws.

\item {\it Parameter Selection}: The increased flexibility comes with the challenge of selecting or learning optimal hyperparameters, which may require additional computational resources (e.g., grid search, Bayesian optimization).

\item The main trade-off is increased mathematical and computational complexity, but this is often outweighed by the gains in flexibility, adaptability, and performance in complex learning tasks.

    \end{itemize}

 \section{Related work on Mirror Descent}\label{Sec3}

% Mirror Descent and Standard Exponentiated Gradient (EG) Update} \label{Sec3}

 \subsection{Notations and Assumptions} In this paper, we shall adopt the mathematical notations typically used in ML and optimization theory contexts.

 Let $\bw=(w_1,\ldots,w_N)^{T}$ be a column vector  where $w_i\in \mathbb{R}$, $\forall i=1,\dots, N$. The vector space of column vectors will be identified with $\mathbb{R}^{N}$; we shall refer to its elements simply as vectors. Let $[x]_+ := \max\{0,x\}$. We denote by $\mathbb{R}^N_+$ the set of vectors with nonnegative entries. Also, we use the notation $\bw^{\alpha}: = [w_1^{\alpha}, \ldots, w_N^{\alpha}]^T$. We denote the scalar product of two (column) vectors $\bv, \bw\in \mathbb{R}^N$  by  $\bv^{T}\bw=v_1 w_1+\cdots v_N w_N$, whereas their Hadamard product is defined as $\bv \odot \bw = [ v_1w_1, \ldots, v_Nw_N ]^T$ .

 %All operation for vectors like multiplications and additions are performed componentwise.
%Let $g:\mathbb{R}\to \mathbb{R}$.  With a slight abuse of notation, we introduce the symbol $g(\bw): = [g(w_1),g(w_2),\ldots, g(w_N)]^T$.

 We shall  denote by $\bw (t)$ the \textit{weight} or \textit{parameter vector} as a function of time $t$.  The gradient\footnote{Whenever there is no ambiguity, we shall use the notation $\nabla$ instead of $\nabla_{\bi w}$} of a differentiable cost function is defined as
 \[
  \nabla_{\bi w} L(\bw) = \partial L(\bw)/\partial \bw: = [\partial L(\bw)/\partial w_1, \ldots, \partial L(\bw)/\partial w_N]^T .
  \]

 %In contrast to generalized logarithms defined later the classical  natural logarithm will be denoted by $\ln (x)$.\\

\begin{assumption}
The learning process is supposed to advance in iterative steps. During step $t$,  we take the weight vector $\bw (t) = \bw_t$ and update it to a new vector $\bw(t+1) = \bw_{t+1}$.
\end{assumption}

\subsection{The Bregman Divergence} Let $F: \mathcal{C}\to \mathbb{R}$ be a twice-differentiable, strictly convex function  defined on a convex domain $\mathcal{C}\subseteq \mathbb{R}^{N}$.  The Bregman divergence $D_f:\mathbb{R}^{N}\times \mathbb{R}^{N} \to \mathbb{R}$ is defined to be the function \cite{Bregman1967}
\beq
D_f(\bu || \bw) = F(\bu) - F(\bw) -  (\bu-\bw)^{T} \bof(\bw), \qquad \bu, \bw\in \mathcal{C}.
\label{Bregman1}
\eeq
Here $F$ is the \textit{potential function}, while $\bof: \mathbb{R}^{N}\to \mathbb{R}^{N}$, defined by $\bos{f}(\bw):= \nabla_{\bi w} F(\bw)$, is referred to as the \textit{link function} or \textit{mirror map}.
%
%The Bregman divergence can be understood as the first-order Taylor expansion of $F$ around  $\bw_1$  evaluated at $\bw_2$.

The Bregman divergence $D_f(\bu || \bw )$  arising from the potential function $F(\bw)$ can be viewed as a measure of curvature. This family of divergences includes many well-known divergences commonly used in applicative contexts, e.g., the squared Euclidean distance, the Kullback-Leibler divergence (relative entropy), the Itakura-Saito distance, the beta divergence and many more \cite{Cia3}.

The gradients of  the Bregman divergence $D_f(\bw || \bw_t)$  w.r.t. the first and second arguments  yield
\beq
\nabla_{\bw} D_f(\bw || \bw_t)= \bof(\bw) - \bof(\bw_t), \qquad \nabla_{\bw_t} D_f(\bw || \bw_t)= -  \bH_f(\bw_t) (\bw - \bw_t),
\label{Bregman2}
\eeq
where
$
\bH_f(\bw_t) =\nabla^2_{\bi w_t} F(\bw_t)$ %= \frac{\partial^2 F(\bw_t)}{\partial \bw^2_t}= \frac{\partial f(\bw_t)}{\partial \bw_t},
 is the Hessian of $F(\bw_t)$ evaluated at $\bw_t$.  In the subsequent analysis,  by properly selecting the range of hyper-parameters present in the link functions, we shall assume that the Hessian matrices will be diagonal, positive-definite  matrices with  positive entries.

\subsection{The Optimization Problem}
Our central objective is to study the following constrained optimization problem:
\beq
	\bw_{t+1} = \rm{arg}~\rm{min}_{\bi w \in \Rp } \left\{ L(\bi{w} )+ \frac{1}{\eta} D_f(\bi{w} || \bi{w}_t)  \right\} \;\; \text{s.t.} \;\; \bi{w}_t\in \Rp,\; \forall t , \;\; ||\bi{w}||_1=\sum_{i=1}^{N} w_{i}=1,
	\label{Eq-1a}
\eeq
where  $L(\bw)$ is a differentiable loss function,  $\eta > 0$ is the learning rate and
$D_f(\bw || \bw_t)$ is the Bregman divergence \cite{Bregman1967}. Our aim is to solve it in a very general way, in a group-theoretical framework.

\subsection{Mirror Descent }
We shall analyze now some general properties of Mirror Descent algorithms, useful in the subsequent analysis.
\begin{remark}
From now on, we shall always assume that the potential function $F$ has a \textit{separated form}: $ F(\bw)=\sum_{i=1}^{N} h(w_i)$.
Consequently, we shall adopt the following notation for the link function:
\beq
\bos{f}(\bw)= \big[f(w_1),\ldots, f(w_n)\big]^{T},
%\log_{G} (\bw):=(\log_{G} (w_1), \ldots, \log_{G} (w_n))^{T}, \qquad \exp_{G} (\bw):=(\exp_{G} (w_1), \ldots, \exp_{G} (w_n))^{T}
\eeq
where $f:\R \to \R$  given by $f(w_i):= \frac{d h}{d w_i}$ is assumed to be an invertible function.
In particular, we introduce the \textit{vector logarithm}, defined as $\bln$$(\bw):=[\ln(w_1),\ldots, \ln(w_N)]^T$ and the \textit{vector exponential}, given by $\bexp$$(\bw):=[\exp(w_1),\ldots, \exp(w_N)]^T$. %In other words,  we associate with a function $f:\mathbb{R}\to \R$ a new vector-valued function, that we denote with the same symbol $f$, with an abuse of notation.
\end{remark}

Computing the gradient of the r.h.s. of Eq. (\ref{Eq-1a}) and setting it to zero  at the minimum $\bw_{t+1}$ yields the so-called prox or implicit MD update
\beq \label{MDimplicit}
	\bof(\bw_{t+1}) = \bof(\bw_t) - \eta \nabla_{\bi w} L(\bw_{t+1}).
	\eeq
% or, equivalently,
%\be
%	\bw_{t+1} = f^{(-1)} \left[ f(\bw_t) - \eta \nabla_{\bi w} L(\bw_{t%+1})\right],

%\ee

Under the assumption $\nabla_{\bi w} L(\bw_{t+1}) \approx \nabla_{\bi w} L(\bw_{t})$, the implicit iteration in Eq. \eqref{MDimplicit} can be approximated by the explicit update. In this way,  we obtain for the normalized (scaled) MD the  form \cite{MD1,shalev2011}
 \begin{empheq}%[box=\fbox]
 {align}
\bw_{t+1} &= \bof^{(-1)} \left[ \bof(\bw_t) - \eta \nabla_{\bi w} L(\bw_t)\right] \label{eq:27} \\
\bw_{t+1} &\leftarrow \bw_{t+1} /||\bw_{t+1}||_1,
		\label{f-1fDT}
\end{empheq}
 where $\bof^{(-1)}(\bw):= \big[f^{-1}(w_1),\ldots,f^{-1}(w_n)\big]^{T}$ is inverse of the link function $\bof(\bw)$ with respect to the composition operation acting component-wise.
%

%Projected gradient descent is a special case of mirror descent using the mirror map $F(\bw) = ||\bw||^2  $.

The differential equation for the continuous-time mirror descent update (CMD), or mirror flow,  as $\Delta t \rightarrow 0$ is given by \cite{MD1,EGSD}
\begin{align} \label{f-1fCT}
	\frac{d\,\bof\!\left(\bw(t)\right)}{d t}= - \nabla_{\bi w} L(\bw(t)).
\end{align}
%
%where
%$\mu = \eta/\Delta t >0$ is the learning rate for the continuous-time learning, and
%$\bof(\bw) = \nabla_{\bi w} F(\bw)$ is a suitably chosen link function \cite{MD1,EGSD}
An alternative CMD update can be written in the form
%Hence, we can obtain an alternative continuous-time MD update in general form:
\beq \nn
 \frac{d\,\bw}{d t} =   -  \; [\nabla_{\bi w}^2 F(\bw)]^{-1} \; \nabla_{\bi w} L(\bw).
\eeq
%

%\end{remark}
\noi Using the  chain rule, the mirror flow can be expressed as
 \beq \nn
\frac{d\,\bof\!\left(\bw \right)}{d t} =
%\frac{d\,f(\bw)}{d \bw} \odot  \frac{d\,\bw}{d t} = \diag\left\{\frac{d\,f(\bw)}{d \bw}\right\} \frac{d\,\bw}{d t}=
- \nabla_{\bi w} L(\bw(t)).
%= \diag\left\{\frac{d\,f(\bw)}{d \bw}\right\} \frac{d\,\bw}{d t}
%=  -\mu \nabla_{\bi w} L(\bw_t)).
\label{chainrule}
\eeq
Thus, an alternative continuous-time MD update, commonly known as mirrorless MD (MMD), can be expressed in the following general form:
\beq
 \frac{d\,\bw}{d t} =   - \rm{diag} \left\{\left(\frac{d\,\bi f(\bi w)}{d \bi w}\right)^{-1}\right\} \nabla_{\bi w} L(\bi w_t)= - \; [\nabla^2 F(\bi w)]^{-1} \; \nabla_{\bi w} L(\bi w(t)) .
\eeq
In discrete time, the normalized MMD update (scaled to have unit $\ell_1$ norm) is given by:
\begin{align}
\bw_{t+1} =& \left[\bw_t  -\eta~ \rm{diag} \left\{\left(
\frac{d\,\bi f(\bi w_t)}{d \bi w_t}\right)^{-1}\right\} \nabla_{\bi w} L(\bi w_t)\right]_+ \label{eq:33}\\ 
\bw_{t+1} &\leftarrow \bw_{t+1} /||\bw_{t+1}||_1 ,
		\label{diagMD}
\end{align}
where
\[ \displaystyle \rm{diag} \left\{ \left(\frac{d\,\bi f(\bi w)}{d \bi w}\right)^{-1}\right\} := \rm{diag} \left\{ \left(\frac{d\,\bi f(\bi w)}{d w_1} \right)^{-1}, \ldots, \left(\frac{d\,\bi f(\bi w)}{d w_N}\right)^{-1} \right\}.
\]
%\be
%	\bw_{t+1} = \bw_t  -\eta \diag \left\{\left(
%\frac{d\,f(\bw)}{d \bw}\right)^{-1}\right\} \nabla_{\bi w} L(\bw_t),
%	\ee
Note that the above-defined diagonal matrix can be regarded as the inverse of the Hessian matrix, having positive diagonal entries for a particular set of parameters.

Mirrorless Mirror Descent (MMD) offers a primal-only, potential-free alternative to Mirror Descent (MD) by interpreting it as a partial discretization of Riemannian gradient flow. This formulation allows MD to be extended to any Riemannian geometry, even when the metric tensor is not a Hessian.

"Primal-only" and  "Potential-free"  indicate that MMD doesn't rely on a dual representation of the problem: this allows us to avoid the computational difficulties emerging if the inverse link function  is impossible to express in closed form or our approximation is not sufficiently accurate.

\subsection{The Natural Gradient Descent (NGD)}
The MMD can be regarded as a particular instance of the Natural gradient descent (NGD).

The NGD is an optimization technique grounded in ideas from  information geometry, and works well as an alternative to  stochastic gradient descent (SGD) across many applications \cite{Amari-98}.

The natural gradient leverages the Riemannian structure of the parameter space to refine the gradient search direction. In contrast to Newton's method, natural gradient updates do not rely on a locally-quadratic cost function.  Instead, they extend the conventional gradient by taking into account the curvature of both the loss function and the considered regularization function.

A suitable strictly concave link  function should be chosen to take full advantage of the underlying geometry of the problem. The mirror descent update step involves a coordinate transformation which is tied to the information-geometric dual parameters. In the continuous-time limit, mirror descent can be represented as a Riemannian gradient flow with respect to the Hessian metric induced by the given Bregman divergence.

\subsection {Basic Mirror Descent Updates: Standard Exponentiated  Gradient (EG)}

In the special case in which $ F(\bw) = \sum_{i=1}^N (w_i \ln w_i - w_i)$, with the corresponding  link function $\bof(\bw) = \bln (\bi w)$,  we obtain the (multiplicative) unnormalized Exponentiated Gradient (EGU) updates  \cite{EG} %.
\beq
		\frac{d\,\bln \bi w(t)}{d t}= -  \nabla_{\bi w} L(\bw(t)), \quad \bw(t) \in \Rp, \;\;\; \forall \, t .
		\label{MDEG1}
	\eeq
In this sense, the unnormalized exponentiated gradient  update  (EGU) corresponds to the discrete-time version of the continuous ODE:
\beq
	\bw_{t+1}
	= \bexp \left( \bln (\bi w_t) - \eta_t\, \nabla_{\bi w} L(\bi w_t) \right) \nonumber
	= \bi w_t \odot \bexp \left( - \eta_t \nabla_{\bi w} L(\bi w_t) \right), \;\; \bi w_t \in \Rp, \;\; \forall t
\label{EGU}
\eeq
where   $\odot$  and $\exp$ are component-wise multiplication and component-wise exponentiation respectively and $\eta_t  >0$ is the \textit{learning rate} for discrete-time updates.\\

In many practical applications, for example, in the Markowitz portfolio optimization, we  need to impose, as an additional constraint, that the weights are not only nonnegative, i.e.,  $w_i \geq 0$ for $i=1,2,\ldots,N$, but also normalized to unit $\ell_1$-norm, i.e.,  $||\bw||_1=\sum_{i=1}^{N} w_i =1$ in each iteration step. In such case,  standard EG update can be derived by minimizing the following optimization problem \cite{EG,KW1995,Helmbold98}:
\beq
	J(\bw_{t+1}) =  \hat{L}(\bw_{t})  + \frac{1}{\eta} D_{KL}(\bw_{t+1} \| \bw_{t}) + \lambda \left(\sum_{i=1}^{N} w_{i,t+1}-1\right),
\eeq
where $\lambda >0$ is the Lagrange multiplier and the last term forces the normalization of the weight vector. The saddle point of this function leads to the standard EG algorithm, expressed in scalar form as \cite{EG,Helmbold98}
\beq \label{eq:38}
w_{i,t+1} = w_{i,t} \; \frac{\exp [- \eta \nabla_{w_{i,t}} L(\bw_{t})]}{\sum_{j=1}^N w_{j,t} \exp [- \eta \nabla_{w_{j,t}} L(\bw_{t})]}, \;\; w_{i,t}>0, \quad \sum_{i=1}^{N}w_{i,t} =1, \;\; \forall i,t.
\eeq
Alternatively, we can implement the normalized EG update as follows:
\bea \label{EG}
&& \bw_{t+1}  =  \bexp \left( \bln (\bi w_t) - \eta_t\, \nabla_{\bi w} L(\bi w_t) \right)  \\
 && \hspace{10mm}	= \bw_t \odot \bexp \left( - \eta \nabla_{\bi w} L(\bi w_t) \right), \;\; \bi w_t \in \mathbb{R}^{N}_{+}, \;\; \forall t, \quad \text{(Multiplicative update)} \nonumber \\
&& \bw_{t+1}  \leftarrow 	\bw_{t+1} / ||\bw_{t+1}||_1, \qquad \qquad \text{(Unit simplex projection)}.
\eea

There exists a broad spectrum of admissible mirror maps
$F(\bw)$ or equivalently link functions $\bof(\bw)$, which can be tailored to the underlying geometry of various optimization problems and adapted to distribution of training data. Using mirror descent with a suitably selected link function can yield substantial enhancements in both performance and robustness with respect to noise and outliers.

The principal novelty of our approach lies in the generation of parameterized link functions $f(\bw)$ in a highly flexible manner, leveraging the theory of group entropies. In particular, we will use group logarithms as the fundamental building blocks of our construction.

\section{Group Logarithms and Mirror Duality: A New Infinite family of Mirror Descent  Algorithms} \label{Sec4}

%In this section, we present two step iterations and  two alternative variants of normalized GEG updates.  In one  variant in each iteration the unnormalized  solution $\tilde{\bw}_{t+1}$ is  scaled (normalized) after each iteration step as $\bw_{t+1} = \tilde{\bw}_{t+1}/||\tilde{\bw}_{t+1}||_1$. Alternative variant is simple projection of the vector $\bw_{t+1}$ onto $\ell_1$-norm unit simplex $\tilde{\bw}_{t+1}$ \cite{Cichocki2024}.
In this section, we shall elucidate the intimate relationship between formal group theory and ML by introducing a novel, infinite class of MD algorithms.

\subsection{Group-exponential link functions and GEG updates}
We begin our analysis with the following
\begin{definition}
We introduce the group-logarithm vector  $$\rm{\bf{log}}_G(\bi w):= [\log_{G}(w_1),\ldots, \log_{G}(w_N)]^T$$ and the group-exponential vector $$\rm{\bf{exp}}_{G}(\bi w):=[\exp_{G}(w_1), \ldots, \exp_{G}(w_N)]^{T}.$$
%Here $\log_{G}(\cdot)$ is a group logarithm and $\exp_{G}(\cdot)$ is its inverse function.
\end{definition}
We can formulate now one of the main notions of our theory.
\begin{definition}
The function $\bof_{G}(\bw):= \rm{\bf{log}}_G(\bi w)$ is said to be the group-theoretical link function.
\end{definition}

%\subsection{New Algorithms}
By using the general MD formulas \eqref{eq:27} and \eqref{f-1fDT}, and the fundamental properties described above, we obtain a wide class of
Generalized Exponentiated Gradient (GEG) updates:
%
%\begin{empheq}{align}
%\left\{
%\beq \label{eq:41}	\begin{array}{l}
\bea \label{eq:41}
&& \bw_{t+1} = \bexp_{G} \left[\blog_G(\bi w_t) - \eta_t \nabla_{\bi w} \widehat{L}(\bi w_t)\right] \\ \nn
&& \hspace{10mm} = \bw_t \otimes_G \bexp_{G} \left(-\eta_t \nabla_{\bi w} \widehat{L}(\bi w_t)\right),  \\
  && \bw_{t+1} \leftarrow \frac{\bw_{t+1}}{||\bw_{t+1}||_1}, \;\; \bw_t \in \mathbb{R}^{N}_{+}, \; \forall t
\;\; \text{(Unit simplex projection)},
 %\bw_{t+1} = \; \text {projection on the  unit}\;\; \ell_1-\text{norm simplex} \;(\tilde{\bw}_{t+1}),
% \end{array}
 \eea
where the generalized  multiplication is defined componentwise for two vectors $\bx$ and $\by$ as follows
\beq
 \bx \; \otimes_G \; \bexp_{G} (\bi y) = \bexp_{G} \left(\blog_{G} (\bi x) + \bi y\right), \;\; \bi x\in \Rp. \nonumber
\eeq
The centred gradient is computed as
\beq
	 \nabla_{\bi w} \widehat{L}(\bw_t)=
		\displaystyle
	  \nabla_{\bi w} L(\bw_t)- (\bw_t^T \nabla_{\bi w} L(\bw_t))\, {\bf 1}= \nabla_{\bi w} L(\bw_t)-  \left(\sum_{i=1}^{N} w_{i,t} \frac{\partial L(\bw_t)}{\partial w_{i,t}}\right) {\bf 1}	
	  \label{Grad-Lparc}
\eeq
(see \cite{Cichocki2024} for details).
As we mentioned in the Introduction, MD updates require the generating (potential) function $F(\bw)$ to be twice differentiable and strictly convex on a convex domain $\mathcal{C}\in \R^{N}$. Consequently, the link function $\bof(\bw):=\nabla F(\bw)$ must be componentwise strictly increasing and invertible.

A concave link function, such as those induced by group logarithms, reduces the growth of the associated Bregman divergence and lowers geometric curvature, improving algorithmic stability but potentially slowing convergence. Conversely, convex link functions, including group exponentials, increase curvature and may yield faster convergence..

\vspace{2mm}

Our theory has the advantage of exhibiting a fundamental symmetry, namely the \textit{Mirror Duality} defined above.
In this section, we shall propose MD new algorithms based on this idea.

\vspace{4mm}

%\vspace{2mm}

We believe that Mirror Duality could be crucial in many applicative contexts.

For instance, once applied to online portfolio selection, our generalized EG framework dynamically responds to changing market conditions and investor risk profiles. This, a priori, could lead to improved robustness and convergence and superior empirical returns compared to standard EG methods \cite{Helmbold98,Momentum}.

As a particular instance of our construction, we recover the extensions of MMD/MD updates that have been recently proposed, based on the one-parametric Tsallis or Kaniadakis  logarithm, the Euler generalized two-parametric logarithm \cite{Cichocki2025a} and the family of logarithms studied in \cite{Cichocki2025b,Cichocki2024}.

\subsection{Dual Mirror Descent Algorithms}

In this section, we propose an alternative algorithm to our  GEG update, designed to improve convergence properties and to naturally promote
 sparser solution weight vectors. We refer to this method as \textit{Dual Mirror Descent} (DMD), which can be formulated explicitly as follows:
\beq \label{eq:43}
%\\
\qquad  {w}_{i,t+1} = \left\{
\displaystyle	\begin{array}{l} \left[ \log_G \left( \exp_G (w_{it}) - \eta \, \nabla_{w_{it}} \widehat{L}(\bw_t) \right) \right]_+  \;\; \text{if}\;\;
\exp_G (w_{it}) - \eta \, \nabla_{w_{it}} L(\bw_t)>0 \\
\\
\exp_G \left( \log_G (\bw_{it}) - \eta \, \nabla_{w_{it}} \widehat{L}(\bw_t) \right) \;\; \text{otherwise},
	\end{array}	
	\right.
\eeq
\beq
 w_{i,t+1} \leftarrow  \frac{|{w}_{i,t+1}|}{||{\bw}_{t+1}||_1}, \;\; \bw_t \in \mathbb{R}^{N}_{+}, \; \forall i,t.
 %\bw_{t+1} = \; \text {projection on the  unit}\;\; \ell_1-\text{norm simplex} \;(\tilde{\bw}_{t+1}).
 \eeq

The DMD update have two branches:

\begin{enumerate}
\item
The Dual Update branch corresponds to the "Dual Mirror Descent" scenario, in which the link function is $f = \exp_G$, a convex mapping that increases geometric curvature and accelerates convergence.

\item

The Primal Fallback branch corresponds to the "Generalized Exponentiated Gradient" (GEG) scenario, where  the link function  $f = \log_G$ is concave, making it robust and well-defined even when gradients are large.

\end{enumerate}
\begin{remark}
Note that switching between DMD and GEG occurs only for large positive values of the gradient $\nabla L(\bw_t)$, whereas for relatively small gradient values the algorithm employs only the DMD branch. Consequently, branch switching is rare and occurs only during the first few iterations.
\end{remark}

Our simulations  indicate  that  the  DMD  algorithm is superior for sparse  and very sparse optimization problems, as the clipping operator $[\cdot]_+$ acts similarly to an  ReLU, efficiently selecting nonzero variables.   In other words,    the  ReLU-like clipping operator  $[\cdot]_+$ acts as a hard thresholding mechanism. If the gradient drives a weight $w_i$ towards zero such that
\beq
z_{i,t} =\exp_G(w_{it}) - \eta \nabla_{w_it} (\bw) \le 1,
\eeq
the operator  instantly sets that weight to $0$.  This ensures both non-negativity and sparsity. Moreover, it effectively filters out additive noise below a specific energy threshold, enabling exact support recovery (Intersection over Union = 1.0) even in low-SNR regimes.

Both DMD and GMD algorithms, through generalized logarithms and exponential functions,  act as  powerful local preconditioners. Small weights (which are common in sparse problems) receive "boosted" gradients, enabling the algorithms to move aggressively even when optimization problem is ill-conditioned. Furthermore,  they function as denoising filters by  effectively truncating the additive noise by rapidly driving the corresponding weights to zero. With appropriately optimized hyperparameters,  they serve  as robust sparsifiers that ignore the "buzz" of ill-conditioning and noise.

\vspace{5mm}

%Remarks:
\begin{remark}
It is worth noting that, using the MMD formulas  \eqref{eq:33} and \eqref{diagMD}, we obtain the corresponding additive gradient descent updates:

\begin{itemize}

\item
By selecting concave group logarithms $\bof(\bw) = \blog_G(\bi w)$ as the link functions, we obtain the following additive gradient update
\begin{align}
\bw_{t+1} &= \left[\bw_t  - \eta_t \, \rm{diag} \left\{\left(
\frac{d\,\blog_G(\bi w_t)}{d \,  \bi w_t}\right)^{-1}\right\} \nabla_{\bi w} \widehat{L}(\bi w_t)\right]_+,\\ 
\bw_{t+1} &\leftarrow  \frac{\bw_{t+1}}{||\bw_{t+1}||_1}, \quad \quad  \bw_t \in \mathbb{R}_+^N,\; \forall t,
		\label{diagMDL}
\end{align}
where $$ \displaystyle \rm{diag} \left\{ \left(\frac{d\,\blog_G(\bi w)}{d \bi w}\right)^{-1}\right\} := \rm{diag} \left\{ \left(\frac{d\,\blog_G(\bi w)}{d w_1} \right)^{-1}, \ldots, \left(\frac{d\,\blog_G(\bi w)}{d w_N}\right)^{-1} \right\} $$ is a diagonal positive definite matrix, which corresponds to GEG updates.

\item
%By choosing the convex  group exponential $\bof(\bw) = \bexp(\bw)$ as link function, we obtain the following additive gradient update
An analogous expression holds for the proposed  Dual MMD update by choosing the link function $\bof(\bw) =\bexp(\bi w)$:
%\begin{empheq}%[box=\fbox]
\begin{align}
\bw_{t+1} &= \left[\bi w_t  - \eta_t \,\rm{diag} \left\{\left(
\frac{d\,\bexp_G(\bi w_t)}{d \, \bi w_t}\right)^{-1}\right\} \nabla_{\bi w} \widehat{L}(\bi w_t)\right]_+, \\ 
\bw_{t+1} &\leftarrow  \frac{\bw_{t+1}}{||\bw_{t+1}||_1}, \quad \quad  \bw_t \in \mathbb{R}_+^N,\; \forall t,
		\label{diagDMMD}
\end{align}
which roughly corresponds to our DMD update.

\end{itemize}
\end{remark}
It is important  to note that the update derived above constitutes a particular instance of  natural gradient descent w.r.t. the Riemannian metric $\bH_F(\bw_t)$, as introduced in its general formulation  by Amari \cite{Amari-98}.

\subsection {MD Updates for Chain Group Logarithms }

Inspired by the theory of group entropies we propose here a novel form of MD which employ not just a single group-theoretical link function $\bof_{G}(\bw)=\blog_{G}(\bi w)$ and its inverse $\bof_{G}^{(-1)}(\bw) =\bexp_G(\bi w)$ but rather a composed link function arising from a \textit{chain} of group logarithms and exponentials \footnote{We recall that in the following analysis, the composition law $\circ$ of vector group logarithms and exponentials is defined componentwise. For instance:
\beq
\blog_{G_1}\circ \bexp_{G_2}(\bi w):= \big[\log_{G_1}\circ\exp_{G_2}(w_1),\ldots, \log_{G_1}\circ\exp_{G_2}(w_n)\big]^{T}.
\eeq}.

% $f_1,f_2, \ldots, f_K$  and their inverses  $f_1^{(-1)},f_2^{(-1)}, \ldots, f_K^{(-1)}$ (each   controlled by on or two hyperparameters).

%Let us consider a group logarithm defined as a chain composition of a set of assigned group logarithms and exponential.
%

\begin{definition}
A chain link function is the composed function defined by
\beq \hspace{5mm}
\bof_{G_1,G_2,\ldots,G_{2n}}(\bw) := (\rm{\bf{log}}_{G_1} \circ \bexp_{G_2}   \circ \cdots \circ \blog_{G_{2n-1}} \circ \bexp_{G_{2n}}  \circ \bln) (\bi w).
\eeq
The group logarithms and exponentials are defined in such a way that $\bof_{G_1,G_2,\ldots,G_{2n}}$ is concave.
\end{definition}
%where $\log_{G_i}$ and $\exp_{G_{i+1}}$ are a set of group logarithms and exponentials ensuring that  $f_{G_1,G_2,\ldots,G_{2n}}$ is concave.

More explicitly, we have component-wise:
\beq \nn
\bof_{G_1,G_2,\ldots,G_{2n}}(\bw)=
\big[\ldots, \log_{G_1} (\exp_{G_2}  ( \cdots (\log_{G_{2n-1}} (\exp_{G_{2n}}(\ln (\bw))))\cdots),\ldots\big]^{T}.
\eeq
By using relations \eqref{Glog} and \eqref{Gexp}, it is easy to prove the following
\begin{lemma}
The relations
\beq
\bof_{G_1,G_2,\ldots,G_{2n}}(\bi w)=\chi(\bln \bi w)= \blog_{\chi}(\bi w)
\eeq
hold, where $\chi=G_1\circ G_2^{-1}\cdots\circ G_{2n-1}\circ G_{2n}^{-1}$.
\end{lemma}
This simple result shows that $\bof_{G_1,G_2,\ldots,G_{2n}}$ is a new group logarithm. For this reason, we shall also refer to $\bof_{G_1,G_2,\ldots,G_{2n}}$ as a \textit{chain group logarithm}. Its inverse function, the \textit{chain group exponential}, reads
\bea \nn
\bof^{(-1)}_{G_1,G_2,\ldots,G_{2n}}(\bw)&=& (\bexp\circ \blog_{G_{2n}}  \circ \bexp_{G_{2n-1}} \circ \cdots \circ \blog_{G2} \circ \bexp_{G_1}) (\bi w)\\
\eea
%or, in components,
%\bea \nn
%f^{(-1)}_{G_1,G_2,\ldots,G_{2n}}(\bw)= \big[\ldots,\exp(\log_{G_{2n}}(\exp_{G_{en-1}}( \cdots \log_{G_2} (\exp_{G_1}(w_i))\ldots))),\cdots]^{T}. \\
%\eea
Hence, the chain MD update can be expressed explicitly in the following form
%
%\be
\beq	\begin{array}{l}
 \bw_{t+1}% = \exp(\log_{G_{2n}}(\exp_{G_{en-1}}( \cdots \log_{G_2} (\exp_{G_1} (\bw_t) - \eta \nabla_{\bi w} L(\bw_t)))))\cdots) = \\
 = \bexp_{\chi}\big(\blog_{\chi}(\bi w_t) - \eta \nabla_{\bi w} L(\bi w_t)\big) \\
 \bw_{t+1} \leftarrow \bw_{t+1} /||\bw_{t+1}||_1 .
		\label{CMD}
\end{array}
 \eeq
%\ee

\begin{remark}
Notice that the formal group law associated with a chain logarithm does not coincide with that of any of its constituents, except for trivial cases. Instead, the chain construction defines a new group law given by $\Phi(x,y)= \chi(\chi^{-1}(x)+\chi^{-1}(y))$.
\end{remark}

\subsection{New Chain link Functions from Group Theory }
% \tcb{(This is a rewritten subsection)}
As an illustrative example, by applying the previous ideas we shall construct a chain link function which combines the Tsallis and Kaniadakis group logarithms and exponentials. Precisely, we have
\begin{eqnarray*}
\hspace{8mm} \log_q(x)&=&\frac{x^{1-q}-1}{1-q}, \qquad \exp_q(x)= [1+(1-q) x]_+^{1/(1-q)}, \qquad q>0 \\  \hspace{8mm}
\log_{\kappa}(x)&=& \frac{x^{\kappa}- x^{-\kappa}}{2 \kappa}, \qquad  \exp_{\kappa}(x)= \left(\sqrt{1+\kappa^2 x^2} + \kappa x\right)^{1/\kappa}, \quad -1\leq \kappa \leq 1.
\end{eqnarray*}
Thus, we define the  link function %\footnote{For simplicity, with an abuse of notation, here we prefer to write $\log_q(\bw):=[\log_q(w_1),\ldots,\log_q(w_N)]^{T}$ and the same for $exp_k(\bw)$, etc.}
\beq \label{lnexpk}
\hspace{15mm} \bof_{q,\kappa}(\bi w) := \blog_{q} \left(\bexp_{\kappa}(\bln\,\bi w)\right), \qquad
\bexp_q(\bi w) =  \left(\kappa \bi w + \sqrt{1+\kappa^2 \bi w^2} \right)^{{1}/{\kappa}}.
\eeq
which is componentwise concave for $q>0, k>0$. Its inverse function reads
\beq \nn
\displaystyle  \bof^{(-1)}_{q,\kappa}(\bw) = \bexp\Big(\blog_{k} \big(\bexp_{q}(\bi w)\big)\Big).
\eeq
Hence, the MD update can be formulated in a  simple and elegant form:
\beq
\bw_{t+1} = \bexp \left(\blog_{\kappa} \big(\bexp_{q}\big[\blog_{q} (\bexp_{\kappa}(\bln(\bi w_t))-  \eta_t \nabla_{\bi w} \widehat{L}(\bi w_t)\big]\big)\right).
\eeq
In the limit $q\to1$, the Tsallis logarithm converges to the standard natural logarithm, and each component of the link function and its inverse take the following forms, respectively:
\beq \nn
\displaystyle f_{1,\kappa} (w_i) =\ln \big(\exp_{\kappa} (\ln w_i) \big) =
\frac{1}{\kappa}\text{arcsinh} (\kappa \ln w_i ), \qquad i=1,\ldots,N
\eeq
\beq \nn
\hspace{10mm} f_{1,\kappa}^{(-1)} (w_i)=\exp \left(\log_{\kappa} (\exp (w_i))\right) = \exp\Big(\frac{\sinh(\kappa w_i)}{\kappa}\Big), \quad i=1,\ldots,N.
\eeq
In this case, the MD update reduces to the exponentiated MD update developed first in \cite{EGSD} and investigated   for several ML applications:
\beq
\displaystyle \bw_{t+1}  = \bexp\Big(\frac{1}{\kappa}\textbf{sinh} \Big[ \textbf{arcsinh} (\kappa \,\bln \,\bi w_t)  -\kappa \eta \nabla_{\bi w} L(\bi w_t) \Big]\Big)
%= \bw_t \odot \textbf{cosh} \left(-\eta \nabla L(\bw_t)\right) + \frac{1}{\kappa} \sqrt{1+\kappa^2 \bw^2_t} \odot \textbf{sinh} \left( - \eta \nabla_{\bi w} L(\bw_t) \right)
.
\eeq
%The analogous formulas for the MSA are left to the reader.

\vspace{2mm}

This simple example illustrates that generating a multi-parameter group logarithm is now straightforward. For instance, it is sufficient to select $\log_{q_1}$, $\log_{q_2}$, $\log_{k_1}$, $\log_{k_2}$ along with their corresponding group exponentials, and then apply the procedure described above to construct a four-parameter instance. Infinitely many additional examples can be obtained by combining different pairs of group logarithms and exponentials.

\section{Computer Simulation Experiments}\label{Sec5}

%\section{Computer Simulation Experiments}\label{sec:experiments}

We evaluate the proposed algorithms---Dual Mirror Descent (DMD),
Generalized Exponentiated Gradient (GEG), and the baseline Standard
Exponentiated Gradient (EG)---on large-scale Simplex-Constrained
Quadratic Programming (SCQP) problems.  The experiments are designed to
assess three critical performance aspects: (i)~convergence speed,
(ii)~sparsity support recovery, and (iii)~robustness to
ill-conditioning and stochastic noise.

\begin{remark}
Throughout this section, we explicitly adopt the single-parameter Tsallis $q$-logarithm and $q$-exponential as our fundamental group-theoretical link functions% (e.g., $\blog_q(\bi w) = \frac{\bi w^{1-q}-1}{1-q}$)
. We employ this simplest possible instantiation of group entropy as a rigorous baseline proof of concept. Demonstrating superior performance with only a single hyperparameter establishes a clear lower bound on the framework’s capability. More sophisticated instances can be explored through computer simulations, which will be the focus of future work.

% However, tuning a single parameter restricts the algorithm to managing only a single family of curvature shapes. By deploying two or more hyperparameters within a  link function, we can effectively decouple the boundary behavior (which governs sparsity induction and noise truncation as $w_i \to 0^+$) from the interior curvature (which acts as a local preconditioner where the majority of active iterates reside). This decoupling affords more flexibility: additional parameters can be leveraged to maintain a suitable sparsity-inducing boundary without excessively steepening the curvature near zero. This represents a pathological condition that otherwise amplifies noise and causes catastrophic floating-point instability.
%
\end{remark}

\subsection{Problem formulation.}
The SCQP problem seeks to minimize a quadratic objective over the
probability simplex:
\begin{equation}\label{eq:SCQP}
  \min_{\bi w\in\Delta_n}\;
  L(\bi w)=\tfrac{1}{2}\,\bi w^\top \bQ\,\bi w
  +\mathbf{c}^\top\bi w,
\end{equation}
where $\Delta_n=\{\bi w\in\mathbb{R}^n\mid\sum_i w_i=1,\;w_i\ge 0\}$,
$\bQ\in\mathbb{R}^{n\times n}$ is a symmetric positive semi-definite
matrix (e.g.\ a covariance or kernel matrix), and
$\mathbf{c}\in\mathbb{R}^n$ is a linear cost vector.

\subsubsection{Challenges: ill-conditioning and noise.}
In realistic large-scale settings ($n\ge 10^3$), two coupled
difficulties arise \cite{Nemirowsky}.

\emph{Spectral ill-conditioning.}
Let $0\le\lambda_{\min}\le\cdots\le\lambda_{\max}$ denote the
eigenvalues of~$\bQ$, and let
$\kappa=\lambda_{\max}/\lambda_{\min}$ be the condition number,
which often exceeds~$10^6$ in practice.
The level sets of~$L$ form highly elongated ellipsoids. Standard
gradient methods oscillate along steep eigendirections while stalling
along flat ones, yielding convergence rates of order
on the order of a per-iteration contraction factor
$\bigl(\frac{\kappa-1}{\kappa+1}\bigr)$ for strongly convex quadratics
(equivalently, a rate $\bigl(\frac{\kappa-1}{\kappa+1}\bigr)^t$ after
$t$ iterations) \cite{Nemirowsky,Beck2003}.

\emph{Stochastic additive noise.}
In practice, the exact gradient is typically unavailable.  Instead, we
observe a noisy estimate
\begin{equation}\label{eq:noisy-grad}
  \hat{\bg}_t = \nabla L(\bi w_t)+\boldsymbol{\xi}_t,\qquad
  \boldsymbol{\xi}_t\sim\mathcal{N}(\mathbf{0},\sigma^2 I).
\end{equation}
Noise prevents standard algorithms from driving inactive weights to
zero; for example, multiplicative EG updates keep each weight hovering
near $w_i\approx\eta\sigma$, destroying sparsity and
interpretability \cite{EG}.

\subsubsection{Benchmark construction.}
To enable large-scale runs without forming dense matrices, we
implement~$\bQ$ in matrix-free form
$\bQ=\bU^\top\!\Lambda\, \bU$,
where $\bU$ is an orthogonal operator (fast orthonormal transform with
random sign flips and permutations) and
$\Lambda=\operatorname{diag}(\lambda_1,\dots,\lambda_n)$ with
$\lambda_i>0$.
Our tests are performed using exponentially decaying eigenvalues.
 
\emph{Spectral normalization.}
We enforce $\|\bQ\|_2=\lambda_{\max}(\bQ)=1$ and prescribe the target
condition number~$\kappa$ by setting
$\lambda_{\min}=1/\kappa$.
This makes the overall curvature scale comparable across experiments
while isolating the effect of ill-conditioning.

\emph{Planted sparse optimizer.}
We plant a $K$-sparse optimum on the simplex by selecting a random
support $S^\star$ ($|S^\star|=K$) and setting $\bw^\star_i=1/K$ for
$i\in S^\star$, $\bw^\star_i=0$ otherwise.  Let
$\bi v=\bQ \bi w^\star$ (computed matrix-free).  We construct
$\mathbf{c}$ to satisfy the KKT conditions with a strict
complementarity margin $\delta>0$ on inactive coordinates:
\[
  c_i=
  \begin{cases}
    -v_i, & i\in S^\star,\\
    -v_i+\delta, & i\notin S^\star.
  \end{cases}
\]
Smaller~$\delta$ makes support recovery harder; in our experiments
$\delta\in[10^{-4},10^{-3}]$.

\emph{Signal-to-noise ratio.}
When noise is enabled, the noise standard deviation is calibrated via
\begin{equation}\label{eq:SNR}
  \mathrm{SNR_{dB}}
  =20\log_{10}\!\Bigl(\frac{\|\nabla L(\bi w_t)\|/\sqrt{n}}
  {\sigma_t}\Bigr)
  \;\;\Longleftrightarrow\;\;
  \sigma_t=\frac{\|\nabla L(\bi w_t)\|}{\sqrt{n}}
  \cdot 10^{-\mathrm{SNR_{dB}}/20}.
\end{equation}
All reported convergence certificates (Frank-Wolfe (FW) duality gap) are
evaluated using the clean gradient $\nabla L(\bi w_t)$.

\subsubsection{Relevance of MD algorithms for SCQP.}
Classical interior-point solvers scale as $O(n^3)$, making them
impractical for large~$n$ \cite{Nemirowsky}.  First-order methods are therefore
essential, but the choice of geometry is paramount.  Projected
gradient descent relies on Euclidean projection onto the simplex,
which is computationally expensive ($O(n\log n)$) and geometrically
unnatural, often causing zig-zagging at the boundary \cite{Beck2003,shalev2011}.  By contrast,
Mirror Descent algorithms equipped with tunable hyperparameters can
adapt the geometry via a Bregman divergence \cite{Nemirowsky,MD1}.

SCQP formulations arise in several modern machine-learning settings.  Recent applications of SCQP
 in sparse representation  focus on balancing computational
 efficiency with structural interpretability,
sparse attention mechanisms in large language models,
robust aggregation in federated learning, adversarial
training, and multiple kernel learning (see for example \cite{JMLRsimplex,wei2025sqp,Yang2022}).
The simplex structure and its
natural sparsity-inducing properties make SCQP an ideal benchmark for
evaluating the proposed Mirror Descent algorithms.

%% ------------------------------------------------------------------
\subsection{Experimental Results}\label{ssec:results}

%\subsubsection{Experimental setup.}
Unless otherwise stated, the following parameters apply:
\begin{itemize}
  \item \textbf{Dimension:} $n\in\{1\,000 - 50\,000\}$.
  \item \textbf{Learning rate:}
        $\eta=1.0$ (Lipschitz-normalized) for all three updates.
  \item \textbf{Sparsity:}
        The number of nonzero elements in the ground-truth solution
        $\bi w^\star$ ranges from 5\% to 70\% of~$n$. The sparsity pattern is planted with $K$  (default $K = 0.1n$) nonzero entries.
  \item \textbf{Condition number:}
        $\kappa=\lambda_{\max}/\lambda_{\min}$ varies from
        $10$ to $10^7$.
  \item \textbf{Noise level:}
        Additive Gaussian noise is injected with SNR ranging from
        $60$\,dB to $-5$\,dB; the default is $\mathrm{SNR}=20$\,dB.
  \item \textbf{Entropic index:}
        Based on preliminary tuning, $q=0.25$ is used for both GEG
        and DMD unless otherwise noted.  This value provides an
        effective balance between convergence speed, the noise-gating
        capability of the $q$-exponential, and numerical stability.
\end{itemize}

\subsubsection{Analysis of Convergence.}
We assess convergence using two complementary measures \cite{Jaggi2013,Bach2015}.

\emph{Relative primal gap.}
For planted benchmarks with known optimizer~$\bi w^\star$ and
$L(\bi w^\star)$, we report
\begin{equation}\label{eq:relprim}
  \mathrm{RelPrimal}(t)
  =\frac{L(\bi w_t)-L(\bi w^\star)}
       {\max \bigl(1,|L(\bi w^\star)|\bigr)}.
\end{equation}

\emph{Frank--Wolfe (FW) duality gap.}
Over the simplex, the FW linear minimization oracle selects the
coordinate of minimum gradient.  Let
$\mathbf{g}=\nabla L(\bi w)=\bQ \bi w+\mathbf{c}$.  The FW
duality gap is \cite{Jaggi2013,Bach2015}:
\begin{equation}\label{eq:FWgap}
  g_{\mathrm{FW}}(\bi w)
  =\langle\bi w,\mathbf{g}\rangle-\min_i g_i,\qquad
  \mathrm{RelFW}(t)
  =\frac{g_{\mathrm{FW}}(\bi w_t)}
       {\max \bigl(1,|L(\bi w_t)|\bigr)}.
\end{equation}
For convex~$L$, the FW gap provides a computable optimality
certificate that yields a valid lower bound
$L(\bi w^\star)\ge L(\bi w)-g_{\mathrm{FW}}(\bi w)$,
making it a reliable stopping criterion that does not require knowledge
of the true optimum.

Both gaps are reported in \emph{relative} form (divided by
$\max(1,|L(\bi w)|)$) to ensure scale-invariant evaluation.

As shown in Figure~\ref{fig:convergence}, the DMD algorithm
significantly outperforms both EG and GEG: the EG curve (blue) remains
nearly flat at a relative primal gap of approximately~$10^{-1}$ even
after 200 iterations, whereas DMD reaches~$10^{-6}$ within the same
budget.

\subsubsection{Large-scale scalability.}
We evaluated scalability by increasing the problem dimension from
$n=1\,000$ to $n=50\,000$, with $\kappa=1\,000$ and sparsity fixed
at $K=0.1n$ (10\% of the dimension).  The stopping criterion is defined as the normalized relative Frank-Wolfe duality gap with respect to the initial difficulty:
\[
  \delta_t=\frac{g_{\mathrm{FW}}(\bi w_t)}
  {g_{\mathrm{FW}}(\bi{w}_0)}\le 10^{-4},
\]
where $g_{\mathrm{FW}}(\bi{w}_0)$ is the FW duality gap at the
uniform starting point.  By holding the sparsity profile constant at
$K=0.1n$, the structural complexity scales consistently with~$n$.
The matrix~$\bQ$ is fully dense (via randomized DCT) and spectrally
normalized to $\|\bQ\|_2=1.0$, preventing the objective from vanishing
at high dimensions.

Table~\ref{tab:scalability} summarizes the iterations required to
reach a relative FW duality gap of $\delta_t\le 10^{-4}$.

\begin{table}[h]
  \centering
  \caption{Iterations to convergence
  ($\delta_t = g_{\mathrm{FW}}(\bi{w}_t)/g_{\mathrm{FW}}(\bi{w}_0)
  \le 10^{-4}$) across problem scales  ($K= 0.1 n$, $\kappa=10^3$).  Values are reported as
  mean\,$\pm$\,std over multiple independent runs.}
  \label{tab:scalability}
  \begin{tabular}{lcccc}
    \toprule
    Algorithm & $n=10^3$ & $n=10^4$ & $n=5\times 10^4$
      & Scaling\\
    \midrule
    DMD ($q=0.25$)  & $124\pm 22$  & $136\pm 9$   & $148\pm 7$
      & Near-constant\\
    GEG ($q=0.25$)  & $458\pm 102$ & $503\pm 61$  & $542\pm 44$
      & Near-constant\\
    EG ($q=1.0$)    & $>5000$      & $>5000$      & $>5000$
      & Stalled\\
    \bottomrule
  \end{tabular}
\end{table}

\begin{figure}[t]
  \centering
  \begin{tabular}{cc}
    \includegraphics[width=0.48\textwidth]{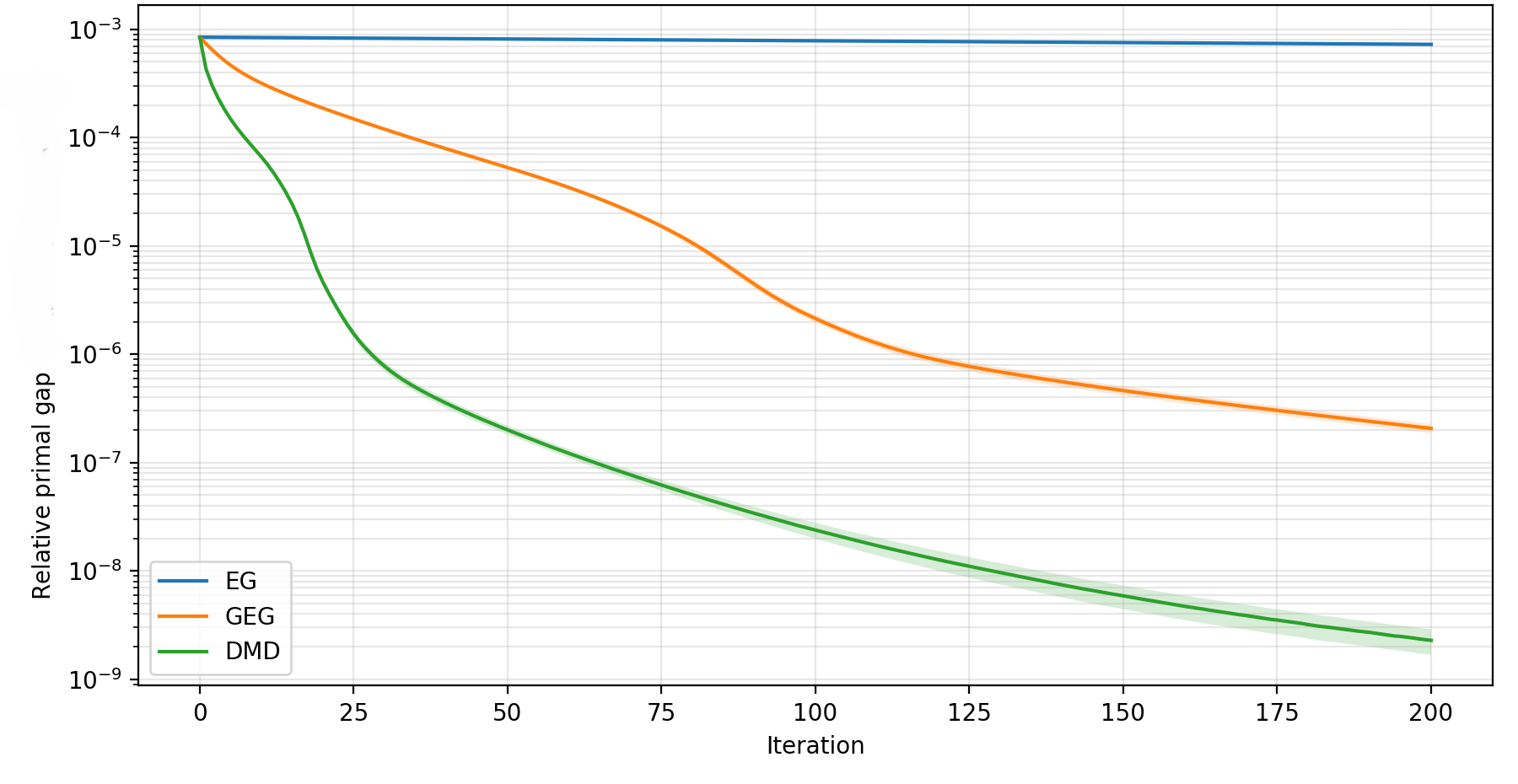} &
    \includegraphics[width=0.48\textwidth]{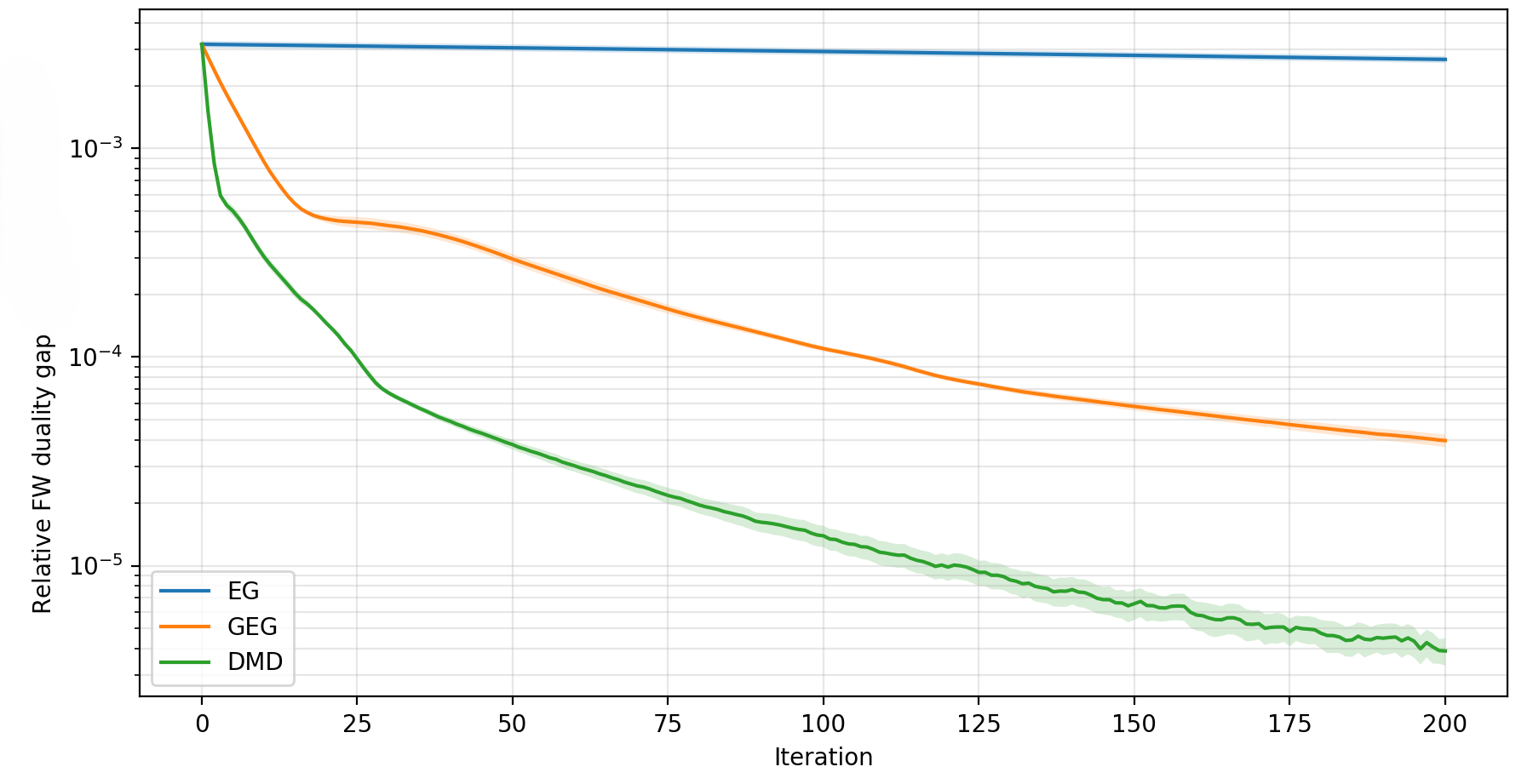} \\
    {\small (a) $n=10^3$, $K=100$, 50 runs, 95\% CI} &
    {\small (b) $n=10^3$, 50 runs, 95\% CI} \\[6pt]
    \includegraphics[width=0.48\textwidth]{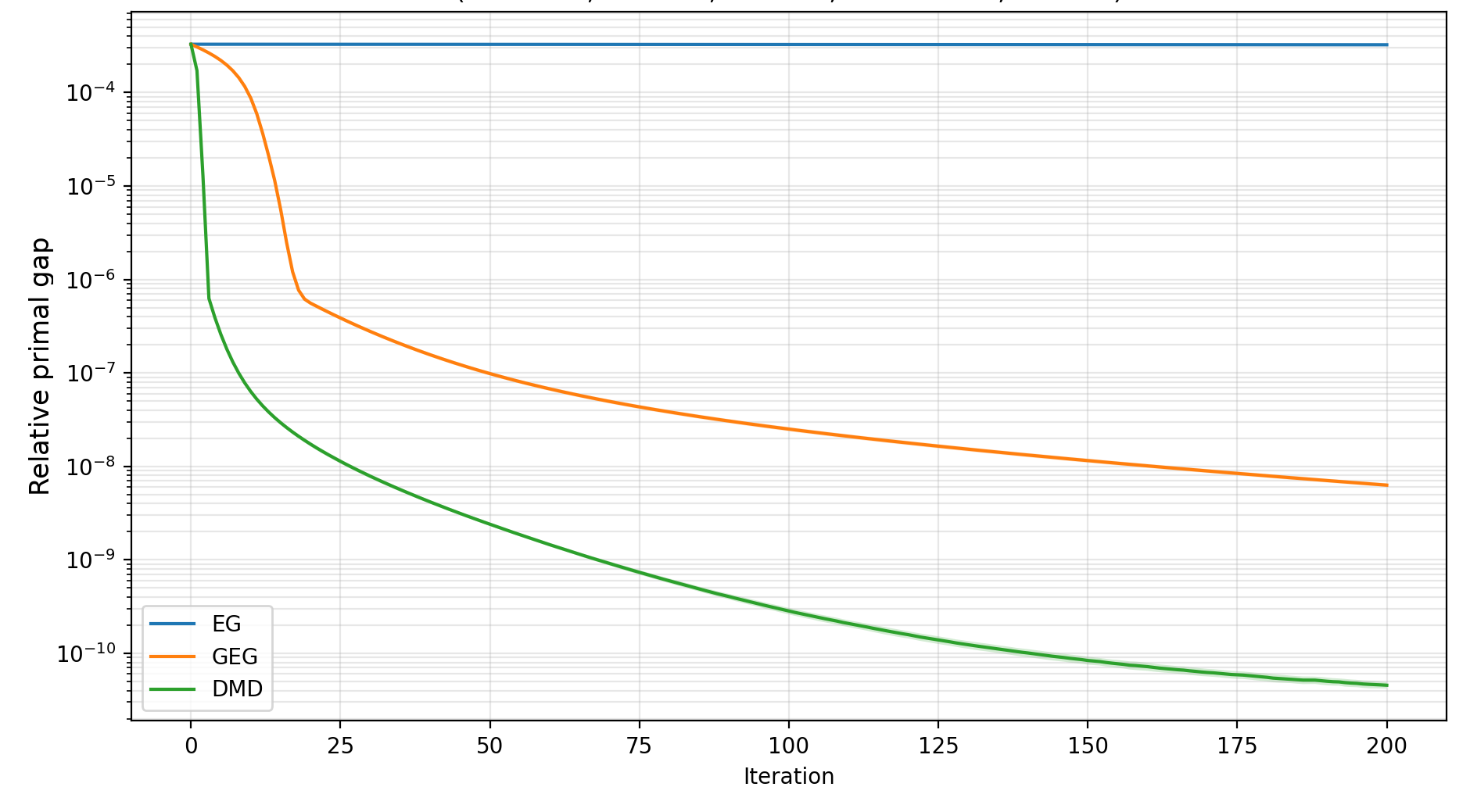} &
    \includegraphics[width=0.48\textwidth]{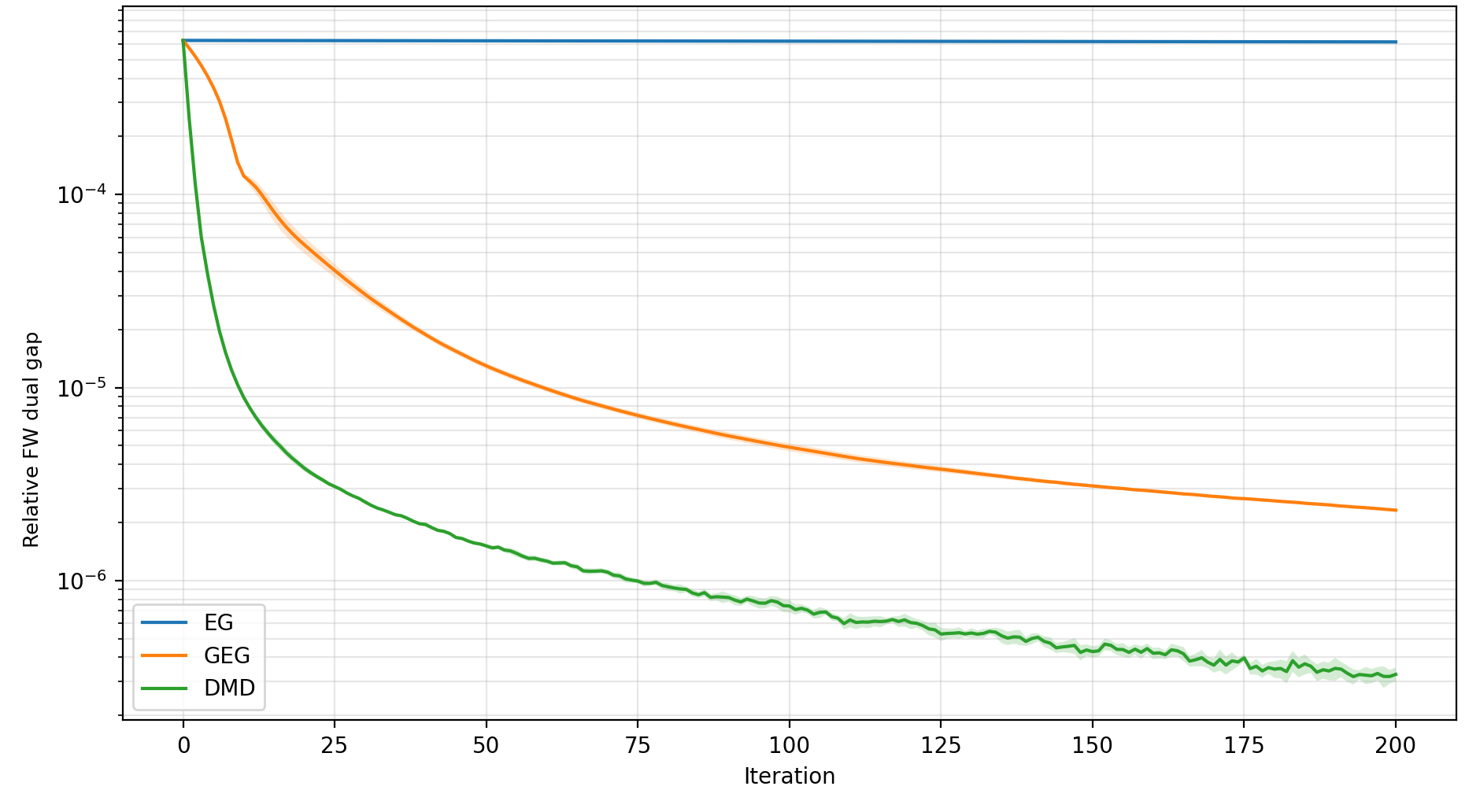} \\
    {\small (c) $n=10^4$, $K=10^3$, 10 runs, 95\% CI} &
    {\small (d) $n=10^4$, $K=10^3$, 10 runs, 95\% CI}
  \end{tabular}
  \caption{Relative primal gap (left column) and relative Frank--Wolfe
  duality gap (right column) versus iteration for EG, GEG, and DMD.
  Benchmarks are matrix-free SCQP instances with spectral
  normalization ($\|\bQ\|_2=1$) and condition number $\kappa=1\,000$.
  The sparsity pattern is planted with $K=0.1n$ nonzero entries and
  the iteration budget is $T_{\max}=200$.  DMD (green) descends
  rapidly, reaching gaps orders of magnitude below those of EG (blue),
  which plateaus due to its inability to drive inactive weights to
  zero. Shaded regions indicate pointwise 95\% confidence intervals
  across runs (mean $\pm 1.96\,\mathrm{SE}$).}
  \label{fig:convergence}
\end{figure}

%\paragraph{Interpretation.}
The results yield several key insights. First of all, we have a sort of \textit{dimensionality independence.} The iteration count for DMD and GEG grows only logarithmically with~$n$, confirming that the geometry of the Tsallis potential ($q<0.3$) is the dominant convergence factor, effectively neutralizing the increased dimensionality of the search space. % \cite{MD1}.

%  \item 

  %\textbf{Step-function convergence of DMD.}
  
Also, we notice that since the $q$-exponential has bounded support, DMD performs exact support recovery within the first few iterations.  Once 50--95\% of inactive variables are set to zero, the algorithm optimizes only on the low-dimensional active sub-manifold, leading to rapid convergence.

 % \textbf{DMD versus GEG.}
   An interesting feature of DMD versus GEG is that    DMD consistently requires approximately $3.6\times$ fewer iterations than GEG to reach high precision.  
   
   %This advantage stems from DMD's hybrid switching mechanism: the dual branch ($\exp_q$) enables aggressive steps that exploit the convexity of the link function, while the primal fallback ensures stability near the boundary.

%  \item 

%  \textbf{Failure of standard EG.}

     As is well known, standard EG ($q=1$) fails to reach high precision in  high-dimensional sparse settings \cite{EG,MD1}.  The Shannon-entropy  geometry is too ``soft'' at the boundary, causing inactive  weights to linger as a noise floor.  In contrast, the  $q=0.25$ index induces curvature  $\nabla^2\Phi\propto w^{-0.25}$ at the boundary, acting as an automatic local preconditioner that overcomes spectral ill-conditioning.

%  \item 

%  \textbf{Numerical stability at scale.}
 
 To conclude, we observe that as $n$ increases, the standard deviation of the DMD iteration count \emph{decreases}, suggesting that the algorithm becomes more predictable as the statistical properties of the matrix-free operator concentrate in high dimensions.
%\end{itemize}

%% ------------------------------------------------------------------
\subsubsection{Sparsity recovery and support identification.}
In many applications (e.g.\ sparse portfolio selection \cite{Helmbold98}, compressed
sensing), identifying the correct subset of active variables (the
support~$S^\star$) is more important than obtaining precise coefficient
values.  We quantify support recovery at iteration~$t$ using the
IoU (Intersection over Union), also called the Jaccard index, between the planted
support~$S^\star$ and the estimated support~$\hat{S}_t$ (defined as
the indices of the top-$K$ entries of~$\bi{w}_t$):
\begin{equation}\label{eq:IoU}
  \mathrm{IoU}(t)=\frac{|S^\star\cap\hat{S}_t|}
                       {|S^\star\cup\hat{S}_t|}.
\end{equation}

Figure~\ref{fig:support-recovery} reveals a distinct phase-transition
behavior.  Unlike EG, which asymptotically suppresses inactive weights
but never eliminates them (owing to the strictly positive standard
exponential), DMD exploits the finite support of the $q$-exponential
($q<1$).  This acts as a hard thresholding operator: once the gradient
pushes a weight below a critical threshold, DMD sets it exactly to
zero.  As a result, DMD achieves perfect support recovery
($\mathrm{IoU}=1.0$) in as few as 2--15 iterations (see
Table~\ref{tab:support}), whereas EG often fails to recover the true
support within the iteration budget.

\begin{figure}[t]
  \centering
  \begin{tabular}{cc}
    \includegraphics[width=0.48\textwidth]{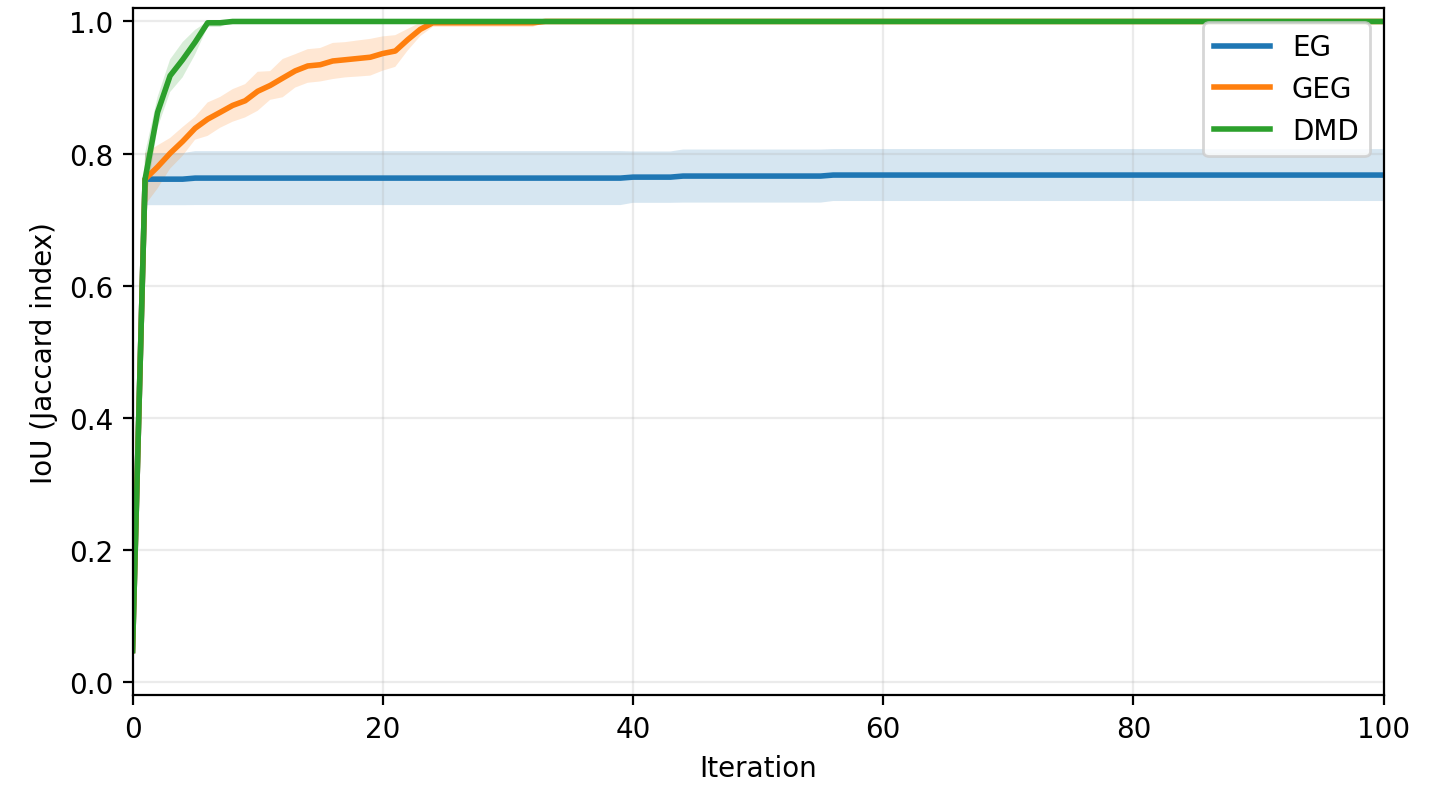} &
    \includegraphics[width=0.48\textwidth]{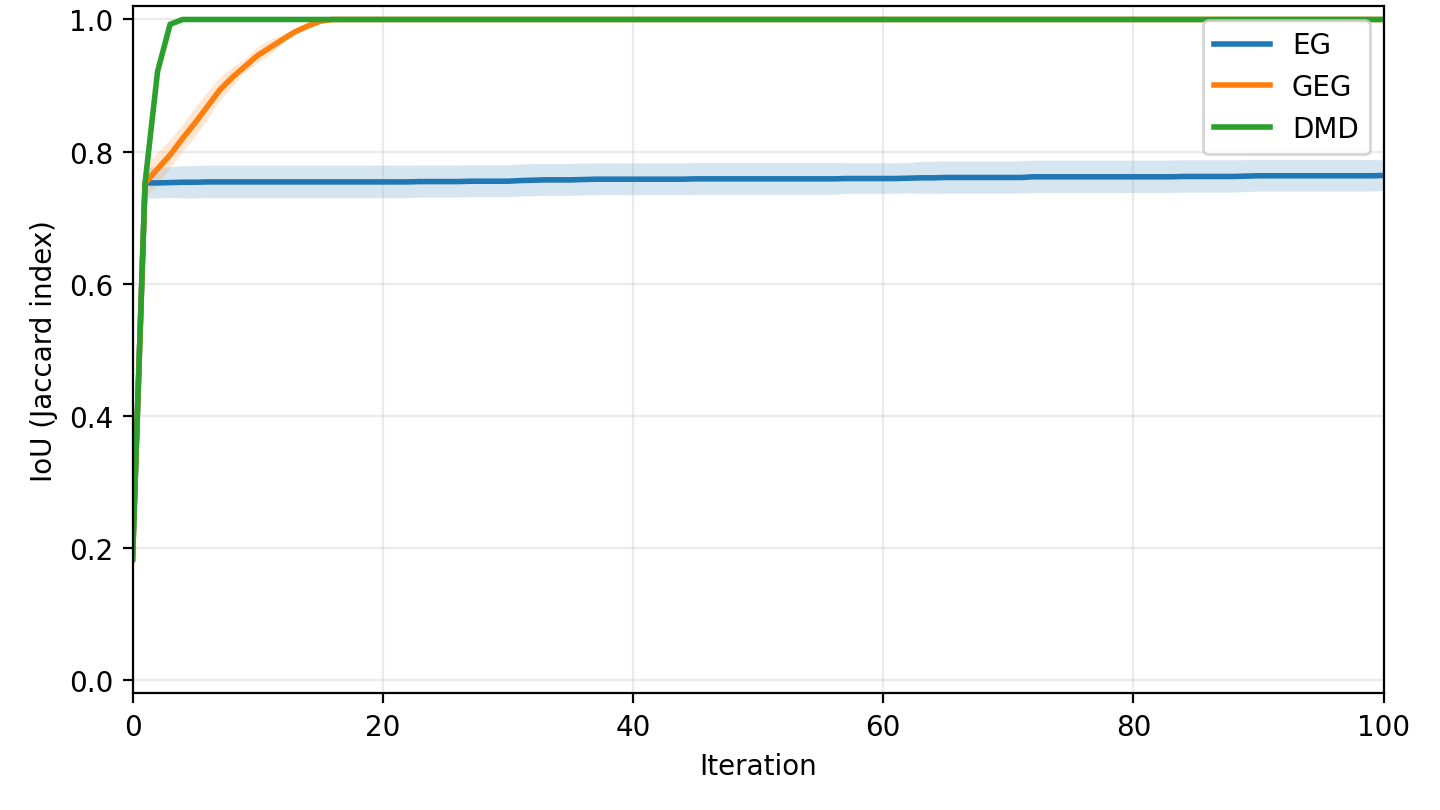} \\
    {\small (a) $K=100$} & {\small (b) $K=300$} \\[6pt]
    \includegraphics[width=0.48\textwidth]{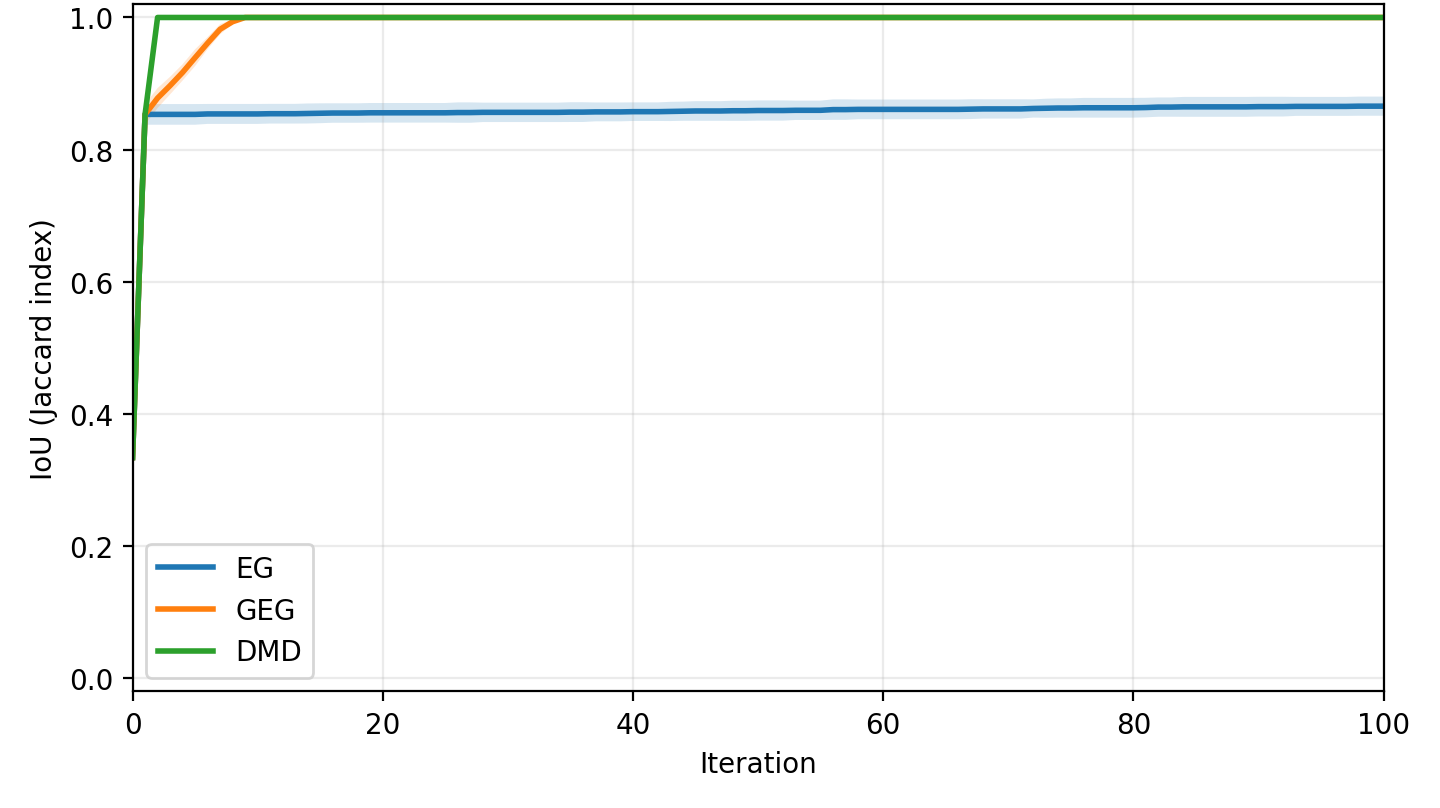} &
    \includegraphics[width=0.48\textwidth]{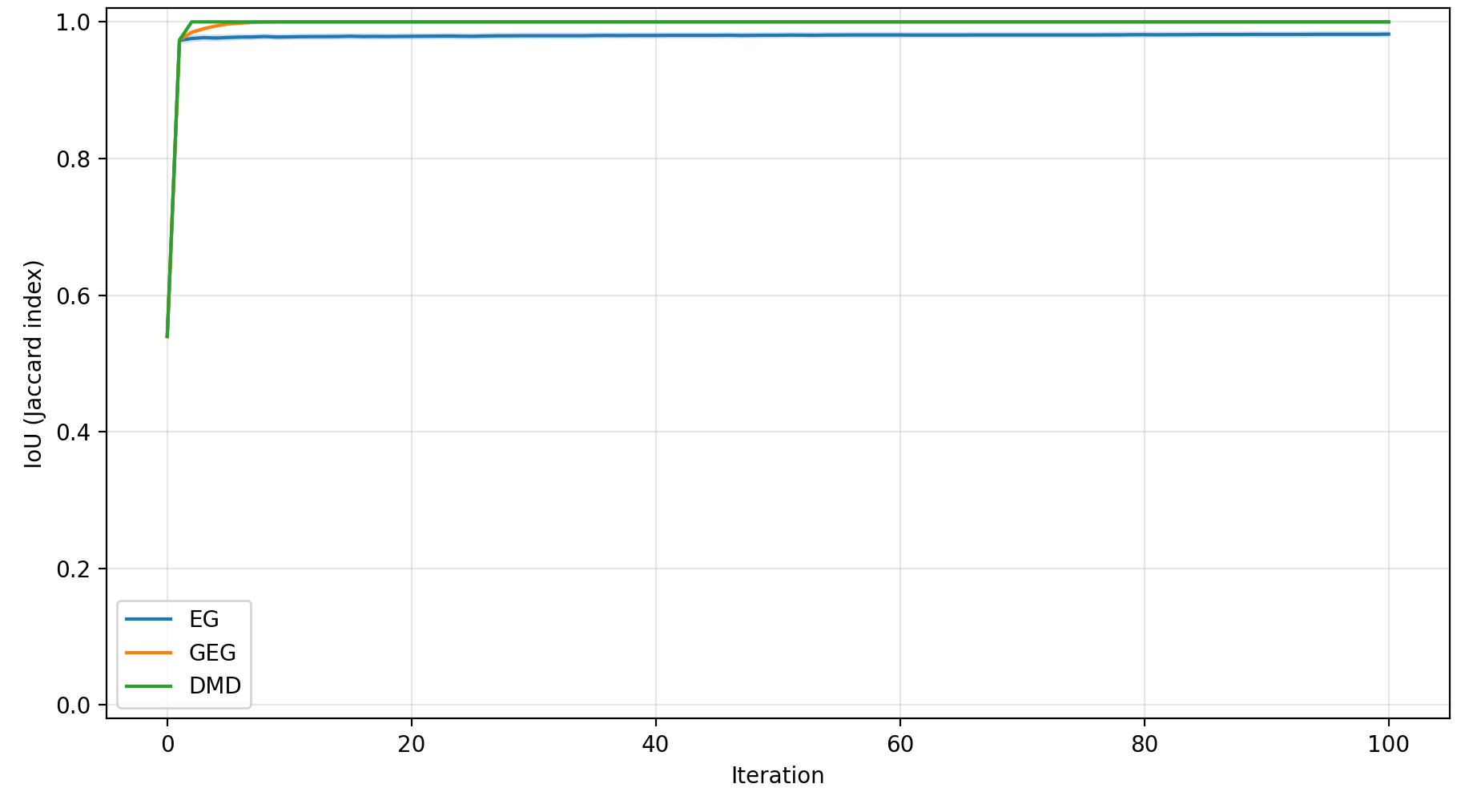} \\
    {\small (c) $K=500$} & {\small (d) $K=700$}
  \end{tabular}
  \caption{Support recovery: Jaccard index (IoU) versus iteration for
  varying sparsity levels
  $K\in\{100,300,500,700\}$ with $n=1\,000$, $\kappa=1\,000$,
  $\mathrm{SNR}=20$\,dB.  DMD behaves as a step-function classifier,
  attaining $\mathrm{IoU}=1.0$ within 2--15 iterations.  EG
  consistently falls short of $\mathrm{IoU}\ge 0.9$ because it
  assigns small nonzero probabilities to inactive elements instead of
  eliminating them.}
  \label{fig:support-recovery}
\end{figure}

Figure~\ref{fig:IoU-vs-FWgap} displays the trajectory of IoU versus
the relative FW duality gap.  The ``vertical cliff'' for DMD and GEG
indicates that these algorithms identify the correct sparsity pattern
(reaching $\mathrm{IoU}=1.0$) \emph{before} achieving high numerical
precision (gap~$<10^{-5}$).  This decoupling of structural
identification from numerical convergence is particularly valuable for
feature-selection tasks.

\begin{remark} We evaluated the  IoU (Jaccard Index)  directly against the relative Frank-Wolfe duality gap to decouple structural identification from numerical precision. In sparse machine learning, correctly isolating the active support often precedes the need for exact numerical convergence. This plot reveals a 'vertical cliff' for the DMD and GEG  algorithms, demonstrating that their non-Euclidean geometry acts as a strict noise-gate: they achieve perfect exact support recovery ($\text{IoU} = 1.0$) early in the optimization trajectory (at gaps $\approx 10^{-3}$), long before wasting computational resources fine-tuning the exact numerical weights of the active manifold. Standard EG, lacking this finite-support thresholding, wastes precision optimizing the noise floor.
\end{remark}

\begin{table}[t]
  \centering
  \caption{Support recovery statistics ($n=1\,000$,
  $\kappa=1\,000$, $\mathrm{SNR}=20$\,dB, 20 runs).  ``Final IoU''
  is measured at the last iteration; ``Iter to IoU$\ge 0.9$'' is the
  first iteration achieving that threshold (mean\,$\pm$\,std).  DMD
  attains perfect recovery in fewer than 10~iterations on average,
  whereas EG fails to converge to the correct support within the
  100-iteration budget.}
  \label{tab:support}
  \begin{tabular}{llcc}
    \toprule
    Algorithm & $K$ & Final IoU & Iter to IoU$\ge 0.9$\\
    \midrule
    DMD & 100 & $1.000\pm 0.000$ & $6.3\pm 0.8$\\
    EG  & 100 & $0.588\pm 0.033$ & $100.0\pm 0.0$\\
    GEG & 100 & $1.000\pm 0.000$ & $26.9\pm 3.8$\\
    \midrule
    DMD & 300 & $1.000\pm 0.000$ & $3.4\pm 0.5$\\
    EG  & 300 & $0.644\pm 0.019$ & $100.0\pm 0.0$\\
    GEG & 300 & $1.000\pm 0.000$ & $14.6\pm 1.0$\\
    \midrule
    DMD & 500 & $1.000\pm 0.000$ & $2.0\pm 0.0$\\
    EG  & 500 & $0.830\pm 0.022$ & $100.0\pm 0.0$\\
    GEG & 500 & $1.000\pm 0.000$ & $5.4\pm 1.0$\\
    \midrule
    DMD & 700 & $1.000\pm 0.000$ & $2.0\pm 0.0$\\
    EG  & 700 & $0.930\pm 0.012$ & $2.0\pm 0.0$\\
    GEG & 700 & $1.000\pm 0.000$ & $2.0\pm 0.0$\\
    \bottomrule
  \end{tabular}
\end{table}

\begin{figure}[t]
  \centering
  \includegraphics[width=0.65\textwidth]{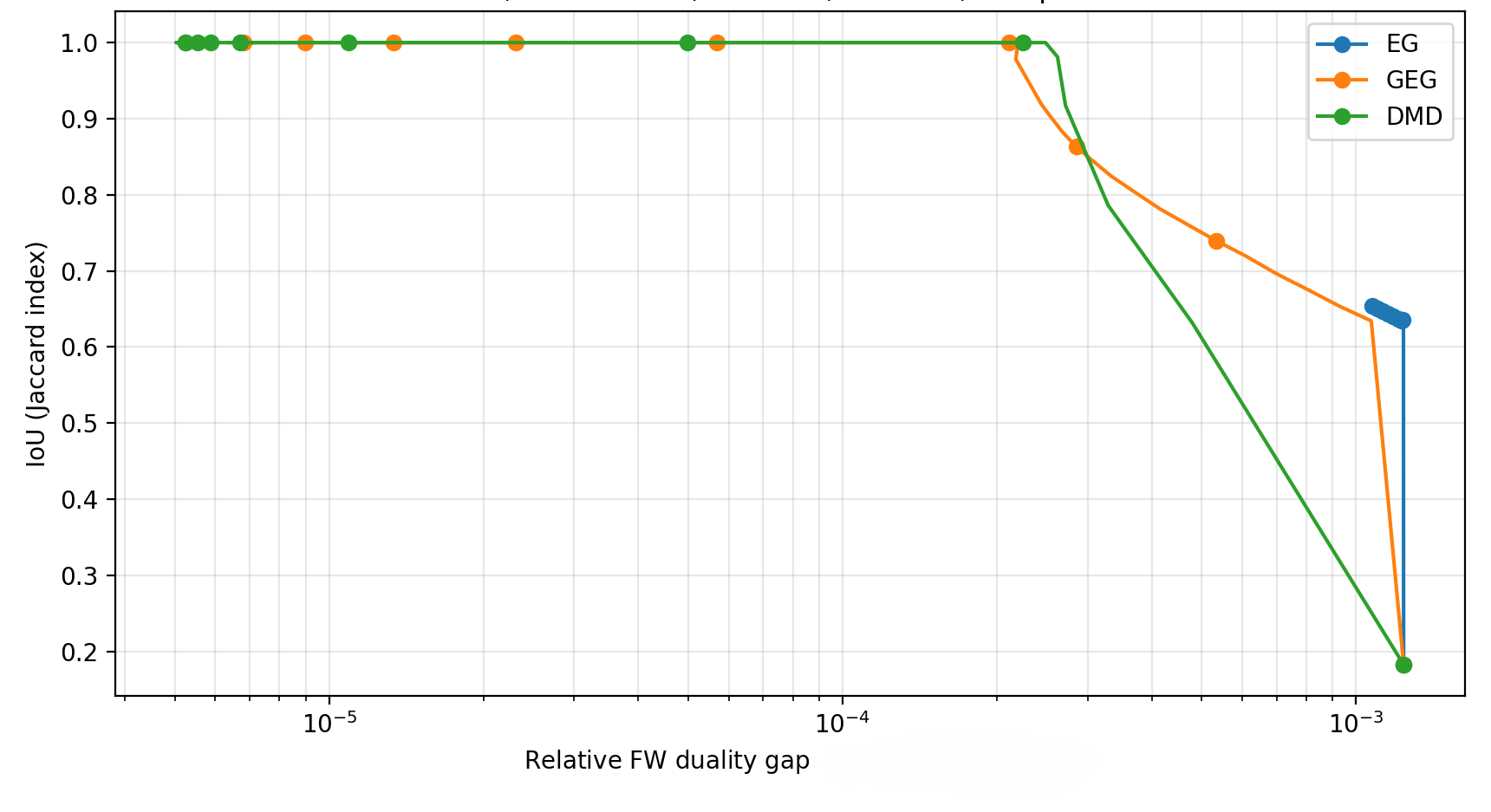}
  \caption{Trajectory of support recovery (IoU) versus the relative FW
  duality gap (log scale).  The vertical cliff for DMD and GEG
  indicates that these algorithms identify the correct sparsity
  structure ($\mathrm{IoU}=1.0$) while the duality gap is still
  relatively large (${\sim}\,10^{-3}$), well before reaching high
  numerical precision.}
  \label{fig:IoU-vs-FWgap}
\end{figure}

\vspace{2mm} 
Let us compare support recovery more closely.   For moderate sparsity ($K=100$--$500$, $n=1\,000$), both GEG and DMD rapidly approach perfect recovery, while EG recovers more slowly and may plateau below the others under noise.  For large~$K$ (e.g.\ $K=700$), the IoU baseline is inherently higher because most indices belong to the support; the convergence curves should be interpreted relative to this baseline.
%\end{itemize}

%% ------------------------------------------------------------------
\subsubsection{Robustness to noise and ill-conditioning.}
We subjected the algorithms to stress tests involving high additive
noise and extreme spectral ill-conditioning \cite{Nemirowsky,Beck2003}.

Figure~\ref{fig:noise-robustness} shows the final relative FW duality
gap after a fixed budget of $T_{\max}=100$ iterations versus SNR
(from $60$\,dB to $-5$\,dB), averaged over 50~independent noise
realizations ($n=2\,000$, $K=200$, $\kappa=1\,000$).  DMD and GEG
maintain low duality gaps even as noise increases (for
$\mathrm{SNR}>5$\,dB), confirming the robustness imparted by the
$q$-exponential map.

\begin{figure}[t]
  \centering
  \includegraphics[width=0.65\textwidth]{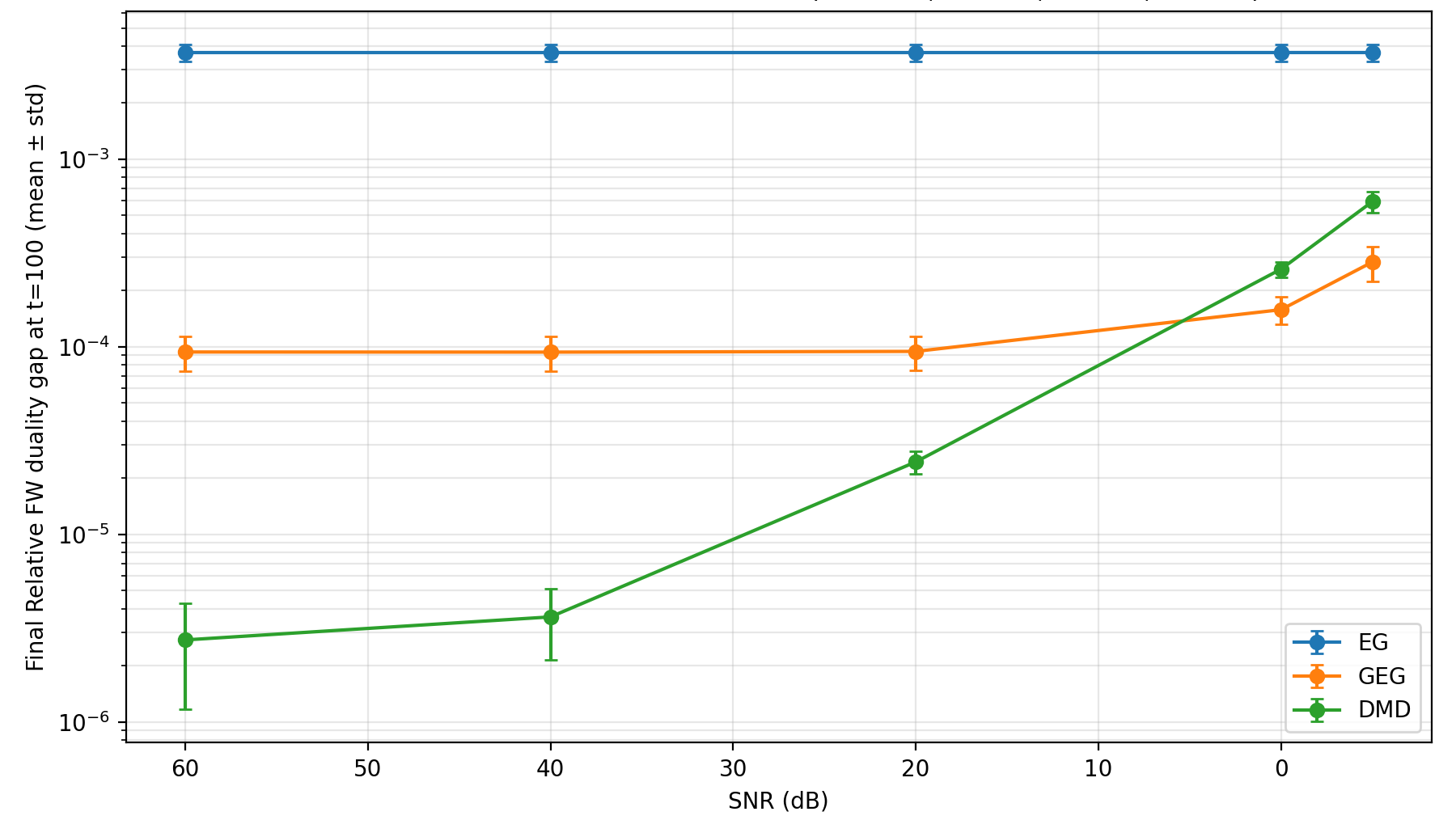}
  \caption{Noise robustness: final relative FW duality gap after
  $T_{\max}=100$ iterations versus $\mathrm{SNR}$ (dB), averaged over
  50~independent noise realizations with dispersion bands.  Setup:
  $n=2\,000$, $K=200$, $\kappa=1\,000$.  DMD and GEG sustain low
  duality gaps for $\mathrm{SNR}>5$\,dB, confirming the noise-gating
  property of the $q$-exponential map.}
  \label{fig:noise-robustness}
\end{figure}

Figure~\ref{fig:ill-conditioning}(a) reports the number of iterations
required to reach the $10^{-3}$ FW duality threshold as the condition
number~$\kappa$ increases up to~$10^7$ ($n=2\,000$,
$\mathrm{SNR}=20$\,dB, 100~iterations).  Both GEG and DMD remain
robust to ill-conditioning even at moderate noise levels.

We also measured the \emph{support recovery delay}, defined as the
minimum number of iterations~$T_{\min}$ required to correctly and
permanently identify the support~$S^\star$.
Figure~\ref{fig:ill-conditioning}(b) shows that DMD and GEG recover
the support structure rapidly across a wide range of condition numbers,
even at $\mathrm{SNR}=-5$\,dB.

\begin{figure}[t]
  \centering
  \begin{tabular}{cc}
    \includegraphics[width=0.48\textwidth]{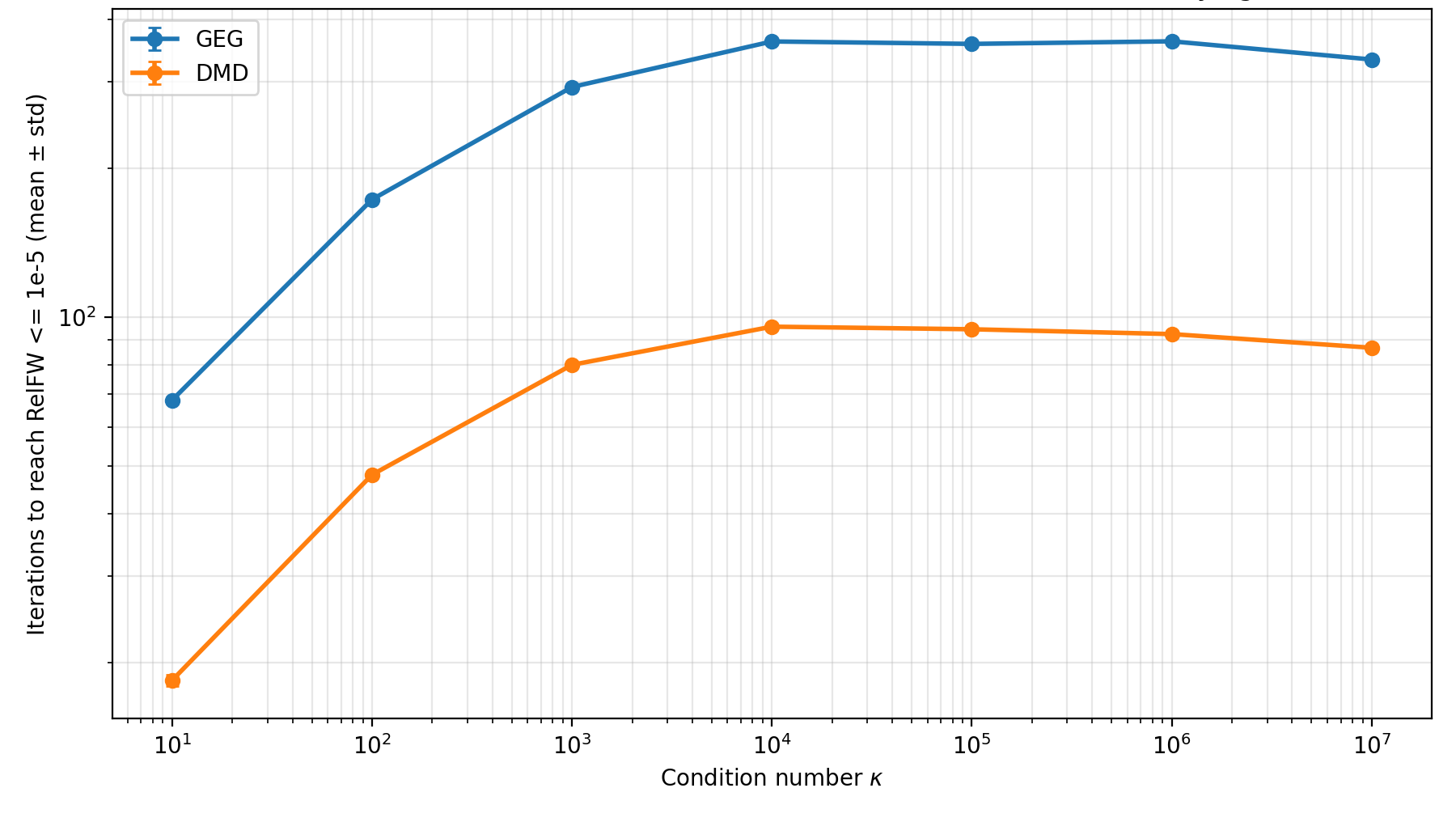} &
    \includegraphics[width=0.48\textwidth]{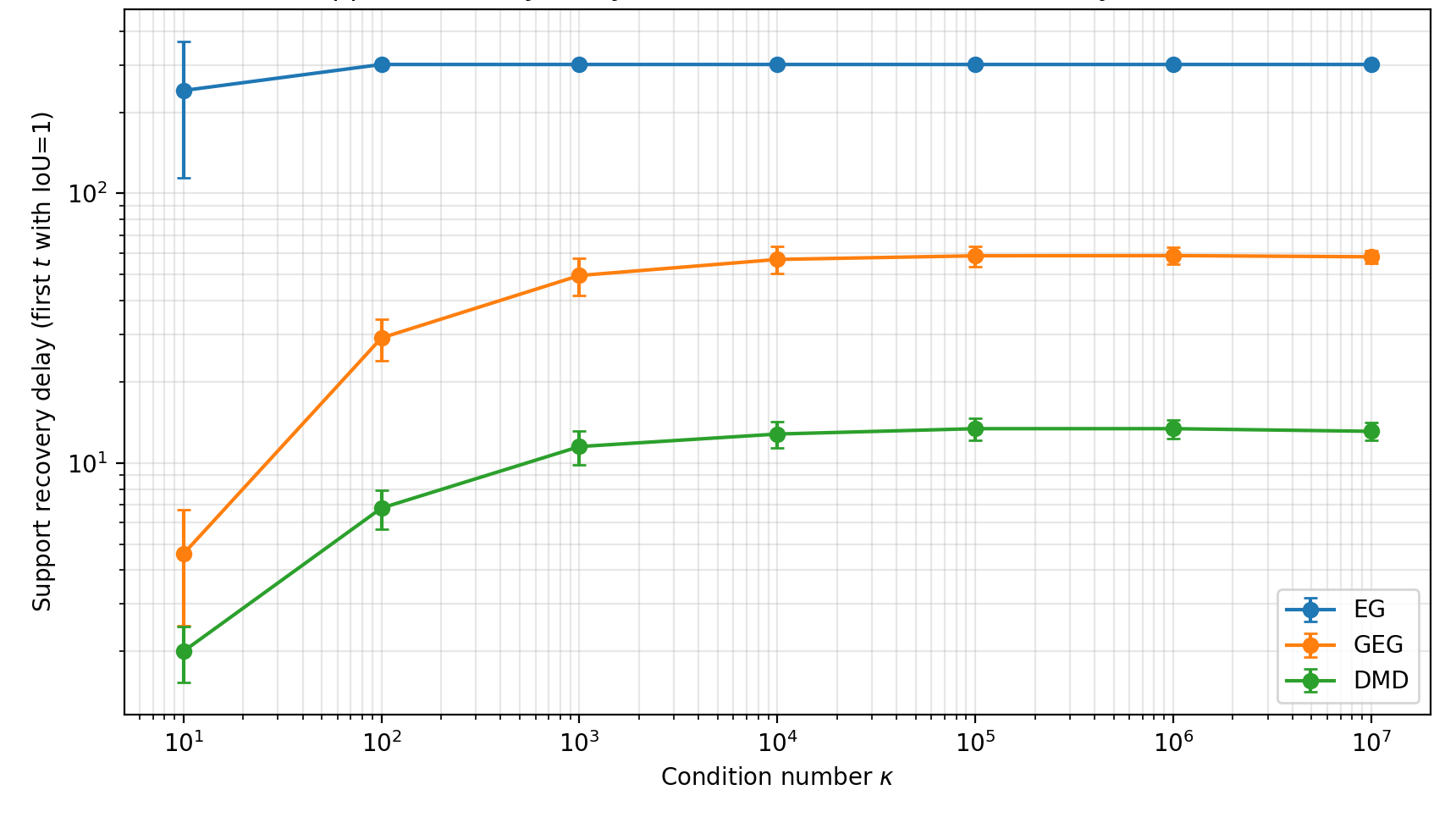} \\
    {\small (a)} & {\small (b)}
  \end{tabular}
  \caption{Robustness to ill-conditioning ($n=2\,000$,
  $\mathrm{SNR}=20$\,dB, 100~iterations).
  (a)~Iterations to reach the $10^{-3}$ FW duality threshold versus
  condition number~$\kappa$ (up to $10^7$).
  (b)~Support recovery delay versus~$\kappa$.
  DMD exhibits remarkable insensitivity to the condition number in both
  metrics.}
  \label{fig:ill-conditioning}
\end{figure}

%% ------------------------------------------------------------------
\subsubsection{Sensitivity analysis of the entropic hyperparameter~$q$.}
We conducted a sensitivity study to evaluate the impact of~$q$ on DMD
and GEG performance.  Simulations were carried out over 50~independent
runs using the matrix-free SCQP benchmark ($n=2\,000$,
$\kappa=1\,000$, $K=200$, $\eta=1.0$).

Table~\ref{tab:q-iters} reports the iterations required to reach a
relative FW duality gap of~$10^{-5}$, and
Table~\ref{tab:q-precision} reports the relative primal gap after a
fixed budget of 100~iterations.  Decreasing~$q$ toward $0.05$
significantly accelerates convergence and improves final precision by
intensifying the sparsity-promoting geometry of the Tsallis potential.

\begin{table}[t]
  \centering
  \caption{Iterations to reach a relative FW duality gap
  $\le 10^{-4}$ as a function of~$q$ ($n=2\,000$, $K=200$,
  $\kappa=1\,000$, 50~runs, mean\,$\pm$\,std).}
  \label{tab:q-iters}
  \begin{tabular}{lcccc}
    \toprule
    Algorithm & $q=0.05$ & $q=0.1$ & $q=0.2$ & $q=0.3$\\
    \midrule
    DMD & $86.2\pm 27.8$ & $98.4\pm 14.2$ & $115.6\pm 18.5$
        & $142.3\pm 25.1$\\
    GEG & $245.5\pm 92.4$ & $312.7\pm 45.3$ & $420.1\pm 62.8$
        & $565.4\pm 88.2$\\
    \bottomrule
  \end{tabular}
\end{table}

\begin{table}[t]
  \centering
  \caption{Relative primal gap after a fixed budget of 100~iterations
  as a function of~$q$ ($n=2\,000$, $K=200$, $\kappa=1\,000$,
  50~runs).}
  \label{tab:q-precision}
  \begin{tabular}{lcccc}
    \toprule
    Algorithm & $q=0.05$ & $q=0.1$ & $q=0.2$ & $q=0.3$\\
    \midrule
    DMD & $6.20\times 10^{-8}$ & $1.25\times 10^{-7}$
        & $8.40\times 10^{-7}$ & $3.55\times 10^{-6}$\\
    GEG & $8.50\times 10^{-5}$ & $4.10\times 10^{-4}$
        & $1.20\times 10^{-3}$ & $6.80\times 10^{-3}$\\
    \bottomrule
  \end{tabular}
\end{table}

%\paragraph{Key observations from the sensitivity analysis.}

From the sensitivity analysis, we can deduce the following features.

%\begin{itemize}
 % \item \textbf{DMD dominance.}
        
        \vspace{2mm}
        
        Across all tested values of~$q$, DMD consistently outperforms GEG in both speed and precision, owing to the high-curvature dual update branch.  For small weights (common in 5--20\% sparsity problems), lower~$q$ provides a more intense local gradient boost, enabling the algorithm to traverse the flat valleys induced by large condition numbers.

%  \item 
  
        The extreme value $q=0.05$ yields the fastest support recovery. However, the increased standard deviation ($\pm 27.8$) indicates a trade-off: while typically best, the algorithm becomes more sensitive to the initial alignment of the gradient relative to the active manifold.

 % \item 
 
% \textbf{Recommended range.}
        
For general-purpose large-scale SCQP, $q\in[0.1,\,0.25]$ provides an optimal balance between aggressive denoising and numerical stability.  Values of $q<0.1$ risk oscillatory behavior under high gradient noise.

  %\item 
%  \textbf{Finite-time extinction.}
       
     Finally, we observe that inactive variables are eliminated almost instantaneously because the support of the $q$-exponential is nearly  rectangular.  The local preconditioning effect ($\nabla^2\Phi\propto \bw^{-q}$) is nearly flat in the interior but acts as an absorbing barrier at the boundary, driving the  solution onto the sparse manifold with maximum speed.
%\end{itemize}

%======================================================================
% REVISED SECTION 6
% Theoretical Analysis of Stability, Convergence, and Robustness
% for DMD and GEG Updates
%======================================================================

\section{Theoretical Analysis of Stability, Convergence, and Robustness for DMD and GEG Updates} \label{Sec6}

In this Section, we shall provide a theoretical analysis of the Generalized Exponentiated Gradient (GEG) and Dual Mirror Descent (DMD) algorithms. We establish a unified framework based on the geometry of Bregman divergences \cite{Bregman1967} to evaluate how the choice of the link function determines the algorithm's stability, convergence speed, and robustness. We show that, for the Tsallis choices considered here, the GEG potential function induced by the $q$-logarithm has unbounded curvature as $w\to 0^+$, whereas the DMD potential function induced by the $q$-exponential has uniformly bounded curvature on $[0,1]$. These curvature bounds are informative for numerical conditioning and step-size sensitivity of the explicit updates; formal convergence rates follow from standard mirror-descent analysis under explicit assumptions on the objective (e.g., bounded gradients and/or relative smoothness/strong convexity).

Our analysis relies on classical results in convex optimization \cite{Nemirowsky,Beck2003} and on the theory of Legendre functions and Bregman projections developed in \cite{BauschkeBorwein1997,BauschkeBorweinCombettes2001}. The notion of relative smoothness and relative strong convexity, as introduced in \cite{LuFreundNesterov2018}, is also instrumental.

\subsection{Formal Properties of Link Functions and the Master Equation}
\label{sec:formal_props}

We analyze optimization strictly over the probability simplex $\Delta_n= \{\bw \in \RR^N_+ : \|\bw\|_1 = 1\}$. The geometry of Mirror Descent is governed by a strictly convex, separable potential function $F(\bw) = \sum_{i=1}^N h(w_i)$. The associated \textit{link function} (or \emph{mirror map}) $\bof: \Delta_n \to \RR^N$ is the exact gradient of the potential function \cite{Nemirowsky,Beck2003,shalev2011}:
\begin{equation}
    \bof(\bw) = \nabla F(\bw) = [h'(w_1), \dots, h'(w_N)]^T.
\end{equation}

We specifically analyze two separable link functions parameterized by the Tsallis index $q \in (0,1)$ \cite{tsallis1988}:
\begin{itemize}
    \item \textbf{DMD} \textit{(Convex Link):} $h'(w) = \exp_q(w) = [1 + (1-q)w]_+^{\frac{1}{1-q}}$.
    \item \textbf{GEG} \textit{(Concave Link):} $h'(w) = \log_q(w) = \frac{w^{1-q}-1}{1-q}$.
\end{itemize}

The fundamental metric governing the algorithmic limits of Mirror Descent is the \textit{Condition Number} $\kappa_F$ of the potential function $F$ \cite{LuFreundNesterov2018}. It is defined as the ratio of the Lipschitz smoothness constant $L_F$ (the maximal curvature) to the strong convexity parameter $\mu_F$ (the minimal curvature). Because the Hessian $\nabla^2 F(\bw)$ is strictly diagonal with entries $h''(w_i)$, these critical parameters are entirely determined by evaluating the scalar second derivative $h''(w)$ over the feasible domain $w \in (0,1]$:
\begin{equation}
    \kappa_F = \frac{L_F}{\mu_F} = \frac{\sup_{w \in (0,1]} h''(w)}{\inf_{w \in (0,1]} h''(w)} .
    \label{eq:condition_number}
\end{equation}
%\begin{itemize}
  %  \item 
    
 We observe that   \textit{stability} strictly requires $\mu_F > 0$, i.e., strong convexity. On the other hand,    \textit{robustness} strictly requires $L_F < \infty$ (bounded smoothness). In order to ensure theoretically a rapid convergence,  $\kappa_F \approx 1$.
%\end{itemize}

%======================================================================
\subsection{Curvature Analysis: DMD vs. GEG}
\label{sec:curvature}
%======================================================================

We start by deriving some elementary geometric properties of the potential functions induced by the respective updates.

\subsubsection{Dual Mirror Descent (Convex Link)}

\begin{lemma}\label{lem:DMD_curve}%.
For $q \in (0,1)$, the local curvature for DMD is given by
\begin{equation}
    h''(w) = [1 + (1-q)w]^{\frac{q}{1-q}}.
\end{equation}
\end{lemma}
\begin{proof}%.
%By the fundamental theorem of calculus, the derivative of the link function defines the second derivative of the potential. 
Applying the standard chain rule to $\exp_q(w)$, we have immediately
\[
h''(w) = \frac{d}{dw} [1+(1-q)w]^{\frac{1}{1-q}} = \frac{1}{1-q} \cdot (1-q) \cdot [1+(1-q)w]^{\frac{1}{1-q}-1} = [1+(1-q)w]^{\frac{q}{1-q}}.
\]
For $q \in (0,1)$, the exponent $\frac{q}{1-q} > 0$, and since $1+(1-q)w > 0$ for all valid $w \geq 0$, it strictly follows that $h''(w) > 0$.
\end{proof}
The following results amounts to say that DMD is a well-conditioned algorithm.
\begin{theorem} \label{thm:DMD_wellcond}
Over the closed simplex domain $\Delta_n$, the DMD potential function satisfies the following properties.
\begin{enumerate}
    \item Strong Convexity: $\mu_F = 1$.
    \item Bounded Smoothness: $L_F = (2-q)^{\frac{q}{1-q}}$.
    \item Condition Number: For all $q \in (0, 1)$, $\kappa_F \le e$.
\end{enumerate}
\end{theorem}
\begin{proof}
\textit{(1) Strong convexity.} Due to the fact that $\frac{q}{1-q}>0$, the curvature $h''(w)$ is a strictly monotonically increasing function of $w$ on $(0,1)$. Consequently, its global infimum on $[0,1]$ occurs at the boundary $w=0$:
\[
\mu_F = \inf_{w \in[0,1]} h''(w) = [1+(1-q)\cdot 0]^{\frac{q}{1-q}} = 1.
\]
Thus, the potential function $F_{\text{DMD}}$ is  guaranteed to be $1$-strongly convex on the simplex \cite{BauschkeBorwein1997}.

\textit{(2) Bounded smoothness.} The global supremum of $h''(w)$ over $[0,1]$ is attained at the upper bound $w=1$:
\[
L_F = \sup_{w \in[0,1]} h''(w) = [1+(1-q)\cdot 1]^{\frac{q}{1-q}} = (2-q)^{\frac{q}{1-q}}.
\]
For any $q\in(0,1)$, this quantity remains strictly finite. For example, at our experimental setting of $q=0.25$, $L_F = 1.75^{0.333} \approx 1.205.$

\textit{(3) Condition number bound.} For the condition number we have $\kappa_F = L_F/\mu_F = (2-q)^{\frac{q}{1-q}}$. By substituting $t = 1-q \in (0,1)$, we obtain:
\[
\kappa_F = (1+t)^{\frac{1-t}{t}} = (1+t)^{\frac{1}{t}-1} = \frac{(1+t)^{1/t}}{1+t} \le \frac{e}{1+t} \le e,
\]
where we invoked the standard limit inequality $(1+t)^{1/t} \le e$. This  proves that the Hessian spectrum of DMD is uniformly, globally bounded on the simplex.
\end{proof}

\subsubsection{Generalized Exponentiated Gradient (Concave Link)}

%\begin{lemma}[Curvature of GEG]\label{lem:GEG_curve}
%The local curvature is given by
%\begin{equation}
%    h''(w) = w^{-q} = \frac{1}{w^q}.
%\end{equation}
%\end{lemma}
%\begin{proof}
Differentiating the Tsallis $q$-logarithm directly yields the local curvature expression:
\beq \label{eq:54}
h''(w) = \frac{d}{dw}\left(\frac{w^{1-q}-1}{1-q}\right) = \frac{(1-q)w^{-q}}{1-q} = w^{-q}.
\eeq
%\end{proof}
The following theorem states that GEG can be ill-conditioned near the boundary.
\begin{theorem} \label{thm:GEG_illcond}
Over the simplex domain $\Delta_n$, GEG satisfies the following properties.
\begin{enumerate}
    \item \textit{Strong Convexity:} $\mu_F = 1$. %.
    \item \textit{Unbounded Smoothness:} $L_F \to \infty$ as $w \to 0^+$.
    \item \textit{Condition Number:} $\kappa_F \to \infty$.
\end{enumerate}
\end{theorem}
\begin{proof}
(1) \textit{Strong convexity.} The function $w^{-q}$ is strictly decreasing for $q>0$. On the interval $(0,1]$, its absolute infimum is located at $w=1$, yielding $\mu_F = 1^{-q} = 1$.

(2) \textit{Unbounded smoothness.} As the coordinate $w\to 0^+$, the term $w^{-q}\to +\infty$ for any $q>0$. Therefore,
\[
L_F = \sup_{w \in(0,1]} w^{-q} = \lim_{w \to 0^+} w^{-q} = +\infty.
\]
(3) \textit{Condition number.} Given $L_F=+\infty$ and $\mu_F=1$, it trivially follows that $\kappa_F = +\infty$.
\end{proof}
According to property $(2)$ of Theorem \ref{thm:GEG_illcond},   the potential function $F_{\text{GEG}}$ is not Lipschitz smooth over the closed simplex. In the terminology of \cite{BauschkeBorweinCombettes2001}, $F_{\text{GEG}}$ is \emph{essentially smooth} (its gradient diverges at the boundary), but not \emph{globally smooth}.

\begin{remark} 
While GEG satisfies the lower-bound requirement for strong convexity, it completely fails the upper-bound requirement for smoothness. The curvature essentially "explodes" near zero. This singularity is the fundamental theoretical source of instability in sparse geometries: as weights are driven toward zero, the gradient of the potential function changes infinitely fast, triggering massive numerical oscillations unless the discrete step size is drastically reduced.
\end{remark}
%======================================================================
\subsection{Step Size Scaling Laws and Stability}
\label{sec:scaling}
%======================================================================

For any explicit (forward-Euler) discretization of the mirror-descent master equation to remain non-divergent, the learning rate $\eta$ must be sufficiently small to respect the local geometry induced by the potential function $F$. A common curvature-based stability guideline for explicit updates is to choose $\eta$ on the order of $1/L_F$ (up to universal constants) \cite{Nemirowsky,Beck2003}.

\begin{proposition}%[Step-size scaling for DMD (curvature-based guideline)] 
\label{prop:DMD_stepsize}
DMD possesses globally bounded smoothness $L_F \le e$. The maximally stable step size is strictly bounded by 
\begin{equation}
    \eta_{\mathrm{DMD}} \le \frac{2}{(2-q)^{\frac{q}{1-q}}} \approx \text{C}
    \end{equation}
where $C$ is a dimension-free constant.

\end{proposition}
\begin{proof}
According to Theorem~\ref{thm:DMD_wellcond}, the global smoothness constant is strictly bounded by $L_F = (2-q)^{q/(1-q)} \le e$. Substituting this into the curvature-based guideline $\eta = O(1/L_F)$ yields an upper bound $C$ that is  independent of the current iteration vector $\bw_t$. 
\end{proof}
Notably, for typical choices $q \in (0, 0.5]$, unit step sizes ($\eta\approx 1$) are often stable with respect to the mirror geometry in our experiments; practical stability still depends on properties of the objective.
For the GEG  case, we have the following result, consequence of its boundary singularity and unbounded smoothness.
\begin{proposition}%[Step-size scaling for GEG (curvature-based guideline)] 
\label{prop:GEG_stepsize}
The maximally stable step size for GEG is dynamically constrained by the smallest active weight:
\begin{equation}
    \eta_{\mathrm{GEG}} \le 2 \, (\bw_{\min,t})^q,
\end{equation}
where $\bw_{\min,t} = \min_i w_{i,t}$.
\end{proposition}
\begin{proof}
For the GEG link function, the local smoothness constant at any arbitrary iteration $\bw_t$ is bounded from below by the smallest coordinate, yielding $L_F(\bw_t) = (\bw_{\min,t})^{-q}$. The maximum allowable step size to avoid immediate geometric divergence is therefore $\eta \le 2/L_F(\bw_t) = 2\,(\bw_{\min,t})^q$.
\end{proof}

This result has a direct consequence. As the GEG algorithm successfully approaches a sparse solution ($w_{\min} \to 0$), its permissible step size shrinks to zero. Attempting to force a constant learning rate will violate the local stability threshold, resulting in chaotic divergence.

\subsection{Convergence Rates and the Necessity of Multi-Parameter Models}
\label{sec:convergence}
%======================================================================

We use standard Mirror Descent Lyapunov analysis for a convex objective $L(\bw)$ with bounded dual gradients $\|\bg_t\|_* \le G$ \cite{Beck2003,Nemirowsky}.

\begin{theorem}[Convergence of DMD] \label{thm:DMD_conv}
With an adaptive step size $\eta \propto 1/\sqrt{T}$, the DMD algorithm achieves the optimal sublinear convergence rate:
\begin{equation}
    L(\bar{\bw}_T) - L(\bw^*) \le \mathcal{O}\left(\frac{1}{\sqrt{T}}\right).
\end{equation}
\end{theorem}
\begin{proof}
The generalized MD descent lemma \cite{Beck2003} establishes the bound
$$L(\bar{\bw}_T) - L(\bw^*) \leq \frac{D_F(\bw^*\|\bw_1)}{\eta T} + \frac{\eta G^2}{2\mu_F}.$$ 
According to Theorem~\ref{thm:DMD_wellcond}, for DMD, $F$ is $L_F$-smooth on $\Delta_n$ , hence
$D_F(\bu\|\bw)\le \tfrac{L_F}{2}\|\bu-\bw\|_2^2$ for all $\bu,\bw\in\Delta_n$. Since $\|\bu-\bw\|_2^2\le 2$ on the simplex, we get $D_F(\bu\|\bw)\le L_F$. Optimizing by setting $\eta = \sqrt{2L_F/(G^2 T)}$ yields the dimension-free $\mathcal{O}(1/\sqrt{T})$ convergence rate. Furthermore, under additional assumptions such as relative strong convexity \cite{LuFreundNesterov2018}, the bounded condition number $\kappa_F$ supports linear (exponential) convergence.
\end{proof}

\begin{theorem}[GEG conditioning on a truncated simplex] \label{thm:GEG_conv}
Let us consider the truncated simplex
$\Delta_n^{\delta}:=\{\bw\in\Delta_n:\min_i w_i\ge \delta\}$, for a fixed $\delta\in(0,1)$. For the GEG
potential function induced by $h'(w)=\log_q(w)$ with $q\in(0,1)$, the curvature satisfies
\(
1 \le h''(w)=w^{-q} \le \delta^{-q}
\) for all $w\in[\delta,1]$. Consequently, on $\Delta_n^{\delta}$ we have
strong convexity $\mu_F=1$, smoothness $L_F=\delta^{-q}$, and condition number
$\kappa_F(\delta)=\delta^{-q}$.
\end{theorem}
\begin{proof}
By eq. \eqref{eq:54}, $h''(w)=w^{-q}$. On $w\in[\delta,1]$, this is bounded below by
$1^{-q}=1$ and above by $\delta^{-q}$ since $w^{-q}$ is decreasing for $q>0$.
Since $\nabla^2F(\bw)$ is diagonal with entries $h''(w_i)$, the stated bounds imply
$\mu_F=\inf_{w\in[\delta,1]} h''(w)=1$ and $L_F=\sup_{w\in[\delta,1]} h''(w)=\delta^{-q}$,
and hence $\kappa_F(\delta)=L_F/\mu_F=\delta^{-q}$.
\end{proof}

The bound $\kappa_F(\delta)=\delta^{-q}$ makes explicit the sensitivity of the GEG geometry near sparse regimes: as iterates approach the boundary (so that an effective $\delta$ becomes small), the curvature scale grows and curvature-based step-size choices become increasingly conservative.

\begin{remark}%[The Bridge to Multi-Parameter Link Functions]
The mathematical limitations exposed above make clear the potential advantages of using multi-parameter link functions. A single parameter $q$ intrinsically couples the boundary sparsity behaviour to the interior curvature. By introducing two or more hyperparameters (e.g., the Euler logarithm or a stretched-exponential logarithm), practitioners can theoretically decouple these effects. These multi-parameter link functions allow us to synthesize a potential function that retains the bounded smoothness of DMD on the interior (enabling large, stable step sizes) while still exhibiting stronger sparsity-inducing behavior near the boundary.
\end{remark}

\subsection{Summary of Theoretical Comparison}
\label{sec:summary}
%======================================================================

\begin{table}[h]
\centering
\caption{Theoretical comparison of the geometric properties of DMD versus GEG over the probability simplex.}
\label{tab:comparison}
\renewcommand{\arraystretch}{1.2}
\begin{tabular}{@{}lll@{}}
\toprule
\textbf{Geometric Property} & \textbf{DMD (Convex Link)} & \textbf{GEG (Concave Link)} \\ \midrule
Curvature Profile $h''(w)$ & Monotonically Increasing & Singular Decay ($w^{-q}$) \\
Global Smoothness $L_F$ & Strictly Finite ($\le e$) & Infinite ($\to \infty$) \\
Curvature-based step-size guideline & $\eta \le \text{Constant}$ & $\eta \le 2 \,(\bw_{\min})^q$ \\
Condition Number $\kappa_F$ & $\le e$ (Optimally Bounded) & $\kappa \to \infty$ (Unbounded) \\
Robustness to Stochastic Noise & High (Maintains Large $\eta$) & Low (Amplified by Curvature) \\
\bottomrule
\end{tabular}
\end{table}

Our analysis demonstrates that DMD enjoys a uniformly bounded condition number on the simplex, which is favorable for numerical conditioning and helps prevent noise amplification in explicit updates. Under additional assumptions (e.g., relative strong convexity), this also supports linear convergence. In contrast, GEG exhibits a boundary curvature singularity, which can make explicit updates increasingly sensitive near sparse regimes and in practice requires more conservative step sizes.

\section{Conclusions and Future Perspectives } \label{Sec7}
%\tcb{(Revised section)}

In this work, we have developed an unified class of generalized  Mirror Descent algorithms and/or generalize multiplicative Exponentiated Gradient algorithms based on the theory of group entropies and formal groups.

In this context, we have introduced the notion of Mirror Duality, which is central in our investigation. It enables us to interchange group-theoretical link functions with their inverses, allowing us to generate the novel class of Dual Mirror Descent algorithms.

This framework generalizes traditional exponentiated gradient methods by exploiting an algorithmic toolkit built from a large family of multi-parametric generalized logarithms and exponentials. The resulting Dual MD updates are inherently flexible, robust, and especially well-suited for machine learning scenarios involving positivity and sparsity constraints, such as the optimization of deep neural networks for classification, clustering, and online portfolio selection. By carefully selecting group-theoretical link functions and tuning their hyperparameters, we  may gain the ability to match the geometry and noise characteristics of any given data set, potentially yielding superior convergence and resilience against outliers in real-world applications.

The utility of generalized group entropies extends significantly beyond simple mirror descent. Their axiomatic composability and rigorous group-theoretic structure pave the way for new forms of information geometry, natural gradient optimization, and decentralized federated learning. Furthermore, the multi-hyperparameter nature of these entropies perfectly aligns with the demands of advanced deep learning tasks where adaptability, precision, and tunability are paramount. Thus, we plan to make extensive use of these customized multi-paramters link functions in our future theoretical research and computational simulations.

We highlight the following research directions for applying group entropies and Dual Mirror Descent in artificial intelligence:
\begin{enumerate}
\item \textit{Regularization and Robustness}: Multi-parametric group entropies can be deployed as highly adaptive proximal regularizers to dynamically control sparsity, generalization robustness, and state-adaptivity in deep learning models, particularly within the unpredictable environments of deep reinforcement learning.
\item \textit{Information Geometry}: Group entropies organically induce new classes of statistical divergences and Riemannian geometries. These can be exploited to construct very efficient natural gradient methods that transcend the limitations of the traditional Kullback-Leibler formulation.
\item \textit{Generalized Loss Functions}: Group entropies automatically provide a rich family of adaptable loss functions. These can be custom-tailored to produce highly robust learning algorithms for complex classification, regression, and generative modeling, especially in domains plagued by extreme heavy-tailed noise or adversarial outliers.
\end{enumerate}

%Looking forward, the mathematical and algorithmic generality of group entropies offers a powerful foundation for innovation in machine learning, especially  reinforcement learning, promising new loss functions, optimizers, and regularizers that can be custom-built for the most demanding machine learning environments.

Looking forward, the mathematical and algorithmic generality of group entropies provides a deep foundation for innovation in machine learning. It promises the systematic development of new, adaptable loss functions, optimizers, and regularizers that can be custom-built for the most demanding machine learning environments.

%\appendix \label{ap:A}

% Acknowledgements and Disclosure of Funding should go at the end, before appendices and references

\section*{Acknowledgment}
The research of P. T. has also been supported by the Project PID2024-156610NB-I00 of Ministerio de Ciencias, Innovaci\'on y Universidades, and by  the Severo Ochoa Programme for Centres of Excellence in R\&D
(CEX2019-000904-S), Ministerio de Ciencia, Innovaci\'{o}n y Universidades y Agencia Estatal de Investigaci\'on, Spain.  P.T. is a member of the Gruppo Nazionale di Fisica Matematica (GNFM) of the Istituto Nazionale di Alta Matematica (INdAM).

% Manual newpage inserted to improve layout of sample file - not
% needed in general before appendices/bibliography.

%\newpage

%\bibliographystyle{plainnat}

%\appendix
%\section{}

%\bibliographystyle{amsalpha}
%\bibliographystyle{amsplain}
\bibliographystyle{plain}
\bibliography{references_finale}

\end{document}